\documentclass{IEEEojcsys}
\usepackage{microtype}
\usepackage{graphicx}
\graphicspath{{./}{../}}

\let\labelindent\relax
\usepackage{user_macros}
\usepackage[capitalize,noabbrev]{cleveref}
\providecommand{\citep}[2][]{\cite[#1]{#2}}
\makeatletter
\AddToHook{begindocument/end}{%
  \ifx\ltx@label\cref@label
    \def\label@in@display@noarg#1{\cref@old@label@in@display{#1}}%
  \fi
}
\makeatother

\hypersetup{colorlinks,urlcolor=blue,linkcolor=blue,citecolor=blue,hypertexnames=false}

\crefname{assumption}{Assumption}{Assumptions}
\Crefname{assumption}{Assumption}{Assumptions}
\crefname{figure}{Fig.}{Figs.}
\Crefname{figure}{Fig.}{Figs.}
\renewcommand{\thetable}{\Roman{table}}
\theoremstyle{plain}
\newtheorem{theorem}{Theorem}[section]
\newtheorem{proposition}[theorem]{Proposition}
\newtheorem{lemma}[theorem]{Lemma}

\theoremstyle{definition}

\newtheorem{assumption}[theorem]{Assumption}
\theoremstyle{remark}
\newtheorem{remark}[theorem]{Remark}

\jvol{XX}
\jnum{XX}
\paper{1234567}
\pubyear{2026}
\receiveddate{XX Month 2026}
\accepteddate{XX Month 2026}
\publisheddate{XX Month 2026}
\currentdate{XX Month 2026}
\doiinfo{OJCSYS.2026.3739678}
\begin{document}
\sptitle{Article}
\title{Infinitesimal Operator World Models: Finite-Sample Guarantees for Offline Reinforcement Learning} 
\author{Nicolas Hoischen\affilmark{1}}
\author{Petar Bevanda\affilmark{1}}
\author{Max Beier\affilmark{1}}
\author{Stefan Sosnowski\affilmark{1}}
\author{Boris Houska\affilmark{2}}
\author{Sandra Hirche\affilmark{1}}
\affil{Chair of Information-oriented Control, Technical University of Munich, Germany.}
\affil{School of Information Science and Technology, ShanghaiTech University, China}
\corresp{CORRESPONDING AUTHOR: Nicolas Hoischen (e-mail: \href{mailto:nicolas.hoischen@tum.de}{nicolas.hoischen@tum.de})}
\authornote{This work was supported by the German Academic Exchange Service (DAAD) program Konrad Zuse Schools of Excellence in Artificial Intelligence, sponsored by the Federal Ministry of Education and Research, and by the European Union’s Horizon Europe innovation action program under grant agreement No. 101093822, ``SeaClear2.0''.}
\markboth{Infinitesimal Operator World Models: Finite-Sample Guarantees for Offline Reinforcement Learning}{NICOLAS HOISCHEN {\itshape ET AL}.}
\begin{abstract}
Continuous-time stochastic processes underlie many natural and engineered systems. In healthcare, autonomous driving, and industrial control, direct interaction with the environment is often unsafe or impractical, motivating offline reinforcement learning from historical data. However, there is limited statistical understanding of the approximation errors inherent in learning policies from offline datasets.
We address this by linking reinforcement learning to the Hamilton-Jacobi-Bellman equation and proposing an operator-theoretic algorithm based on a simple dynamic programming recursion. Specifically, we represent our world model in terms of the infinitesimal generator of controlled diffusion processes learned in a reproducing kernel Hilbert space. By integrating statistical learning methods and operator theory, we establish global convergence of the value function under sufficiently small generator-approximation error and derive explicit finite-sample guarantees with bounds tied directly to system properties such as smoothness and stability.
Our theoretical and proof-of-concept results indicate that operator-based approaches may hold promise for solving offline reinforcement learning using continuous-time optimal control.  
\end{abstract}
\begin{IEEEkeywords}
Learning for Control, Reinforcement Learning, Statistical Learning, Stochastic Optimal Control.
\end{IEEEkeywords}
\maketitle
\section{INTRODUCTION}
A wide variety of phenomena, from the motion of molecules, the nerve activation in our brains, the value of a stock in the financial market, to the dynamics of robotic systems, can be modeled as random processes governed by Stochastic Differential Equations (SDEs)~\cite{sarkka2019applied}. 
Consequently, making decisions in these processes that optimize a quantity of interest is a common task across various disciplines. Examples include the expected return-on-investment in finance, the distance to a goal in robotics, or the yield in chemical engineering. Tools that have been immensely successful are optimal control when accurate reduced models are available or reinforcement learning, when interaction with the environment (or a simulator) is possible. However, some problems do not allow for interaction and have no accurate simulators or surrogate models. 
Offline Reinforcement Learning (RL) attempts to build a decision-making policy given a reward signal from historical observation data. This enables policy learning without requiring online environment interaction.
\begin{figure}[!t]
    \centering
    \includegraphics[width=0.46\textwidth]{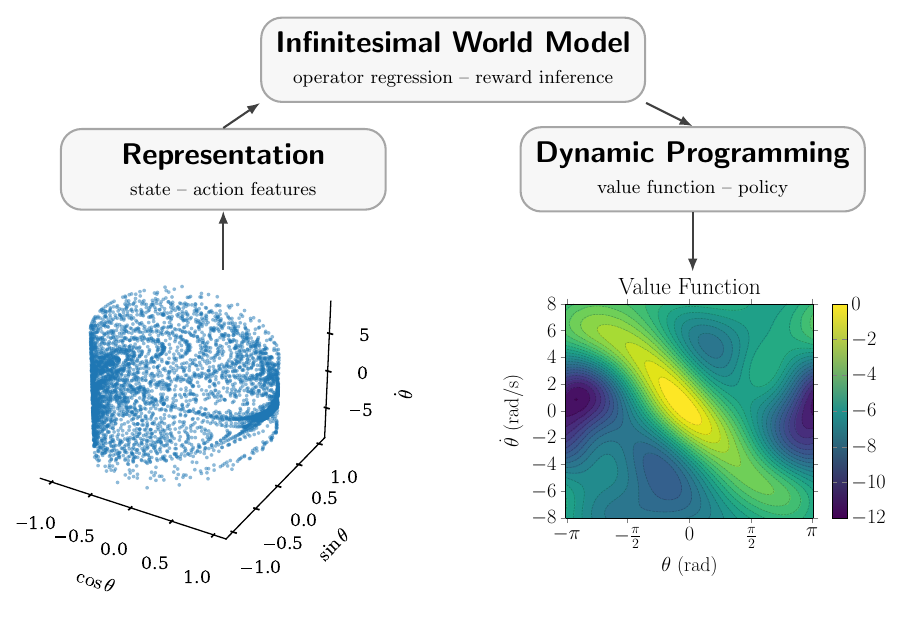}
    \vspace{-1em}
    \caption{Overview of the Operator-Based Continuous-Time Offline Reinforcement Learning ($\textsf{O-CTRL}$) algorithm: a generator world model based on a Reproducing Kernel Hilbert Space (RKHS) representation of state–action data enables dynamic programming for optimal value functions, illustrated on the swing-up pendulum task from \texttt{Gymnasium} \cite{towers2024gymnasium}.}
    \label{fig:pendulumPlot}
\end{figure}
The paradigm seeks to transfer the data-first approach behind breakthroughs in computer vision and natural language to decision-making contexts in which online interaction is costly, time-consuming, unsafe, or infeasible.

\smallskip
Despite substantial empirical progress in offline reinforcement learning \cite{prudencio2023survey}, its theory is well-developed mainly for discrete-time Markov Decision Processes (MDPs)~\cite{yu2020mopo}.
For systems naturally described by SDEs, we still lack an end-to-end learning-theoretic account of offline RL, from data to value-function approximation, including its dependence on problem properties and sample size.
In discrete time, offline RL includes both model-based~\cite{yu2020mopo, kidambi2020morel} and model-free methods~\cite{kostrikov2022offline}. Available continuous-time methods are mostly model-free~\cite{jia2023q} or lack a learning theory. For a comprehensive review of offline RL, we refer to~\cite{prudencio2023survey}.
Under standard technical assumptions, RL can be viewed through the lens of optimal control, with the Hamilton-Jacobi-Bellman (HJB) equation providing a central link to the numerical analysis of optimal control~\cite{munos2000study,doya2000reinforcement}. Existing results that explicitly connect continuous-time RL with stochastic optimal control, however, either focus on linear-quadratic settings~\cite{wang2020reinforcement} or require online interaction with the environment~\cite{jia2022policya}. Numerical schemes for optimal control also typically lack a statistical learning theory relating value-function error to problem regularity and sample complexity~\cite{Lutter2023,halperin2024distributional,Meng2024}. This leaves the nonlinear offline setting without comparable end-to-end guarantees. Direct numerical solution of HJB equations is also limited by the curse of dimensionality, since grid-based state discretizations grow exponentially with the state dimension. Our nonparametric kernel formulation avoids grid-based state discretization by approximating the value function in a sample-induced subspace of a Reproducing Kernel Hilbert Space (RKHS), with computational aspects discussed in~\Cref{sec:implementation-examples}.

\smallskip
\noindent\textbf{Contributions.}
In this paper, we adopt an operator-based perspective to decompose the offline continuous-time RL problem, including world model operator learning and policy optimization. Since the reward is used only after learning the generator, our world-model step follows the reward-free RL paradigm~\cite{wang2020reward}. Our contributions are:
\begin{enumerate}[label=(\arabic*),leftmargin=3.5ex,nolistsep]
    \item We derive $\textsf{O-CTRL}$, a model-based dynamic-programming algorithm for continuous-time offline RL with a reward-free world-modeling stage, in which a single learned operator world model can be reused across rewards without further environment interaction. Each policy-improvement step admits a closed-form unique global maximizer.
    \item We derive a nonparametric estimator for controlled-diffusion generators from generic state-action infinitesimal data, a setting that, to the best of our knowledge, has not been treated at this level of generality.
    \item We establish global convergence of the value iteration and finite-sample guarantees under sufficiently small model-approximation error by combining linear operator learning with convex optimization, yielding bounds that depend on smoothness, stability, discretization error and sample size.    
    
\end{enumerate}

\smallskip
\noindent\textbf{Organization of the paper.}
Section \ref{sec:related-work} discusses related work. Section \ref{sec:background-and-problem} reviews the necessary background and formalizes the continuous offline RL setting. Section IV reviews the optimal control background and its operator formulation underlying the $\textsf{O-CTRL}$ policy update. Section \ref{sec:operator-models} introduces our nonparametric controlled-diffusion generator estimator from generic state-action infinitesimal data. Section \ref{sec:operator-dynamic-programming} combines these ingredients to derive $\textsf{O-CTRL}$. Section~\ref{sec:end-to-end-bounds} analyzes its end-to-end learning rates. Section \ref{sec:implementation-examples} illustrates the theory with numerical examples. The appendix contains the supporting proofs.

\section{LITERATURE REVIEW}\label{sec:related-work}
\subsection{CONTINUOUS-TIME RL AND OPTIMAL CONTROL}
Continuous-time RL is naturally formulated through stochastic optimal control and the HJB equation~\cite{doya2000reinforcement,munos2000study,fleming2006controlled}. This connection has been used to analyze temporal-difference learning, actor-critic methods, and policy-gradient schemes for controlled diffusions~\cite{wang2020reinforcement,jia2022policya,jia2023q}, clarifying the continuous-time limit and the role of optimal control structure. A related line reduces continuous-time problems to discrete-time MDPs, sometimes with adaptive decision intervals to handle irregular observations and reduce sensitivity to a fixed control frequency~\cite{yildiz2021continuous,treven2024sense}. Such reductions can be effective, but remain mediated by a chosen transition discretization and numerical-scheme assumptions~\cite{kloeden1992numerical}. Our formulation instead stays continuous at the modeling level by learning the infinitesimal generator directly and introducing a timestep only later to solve the HJB flow.

\subsection{OPERATOR-BASED \texorpdfstring{WORLD }{WORLD }MODELING}

World models in RL are internal representations of system dynamics used to simulate and predict the outcomes of actions~\cite{schmidhuber1990making,lecun2022path,levine2022understanding}. Operator-theoretic characterizations of dynamics can be interpreted as world models because they encode how observables evolve, thereby allowing expectations relevant to control to be evaluated without explicit simulator rollouts~\cite{novelli2024operator}. Moreover, these operators can be estimated with finite-sample learning guarantees~\cite{li2022optimal,kostic2024consistent}. More generally, operator-theoretic models learn conditional expectations and distributional state evolution in feature spaces, unifying Markov semigroups, transfer operators, infinitesimal generators, and kernel embeddings of dynamical systems~\cite{GrunewalderLBPGPJ2012,Kostic2022LearningSpaces,kostic2024learningGenerator}. This viewpoint fits continuous-time decision problems because policy evaluation and HJB-based control are already expressed through semigroups, resolvents, and generators. A substantial control literature uses Dynamic Mode Decomposition (DMD), Extended Dynamic Mode Decomposition (EDMD), and Koopman approximations for prediction and control, including Koopman-based Model Predictive Control (MPC), and bilinear Koopman or generator models for control-affine systems~\cite{korda2018linear,peitz2020data, goswami2021bilinearization}. These methods enable controller synthesis with lifted finite-dimensional predictors, but the chosen dictionary or lifting architecture becomes part of the model class and can strongly affect approximation quality. In contrast, nonparametric RKHS formulations avoid an explicit finite-dimensional observable dictionary because the kernel induces the hypothesis space and the data generate the finite representation. Related kernel-based Koopman methods with Nyström approximations design Linear-Quadratic Regulator (LQR) controllers on lifted linear surrogate models, with guarantees for the Riccati solution and control objective~\cite{Caldarelli2024}.
In discrete-time RL, related operator constructions support conditional mean embedding Bellman updates and policy mirror descent~\cite{distribBellman,novelli2024operator}, but remain tied to discrete-time Bellman operators or finite action sets~\cite{novelli2024operator}. The closest continuous-time predecessor is the kernel infinitesimal-generator framework of~\cite{bevanda2025}. That work connects generator learning to control, but under a restrictive data-collection protocol separating autonomous and control directions and without an end-to-end perturbation bound for the resulting value function. Here, the controlled generator is learned from generic offline tuples involving state, action, and infinitesimal state change.

\subsection{OFFLINE RL AND STATISTICAL GUARANTEES}
Recent continuous-time RL work gives convergence and regret guarantees for general function approximators~\cite{treven2024sense,zhaosample}, mainly in finite-horizon online settings and without analyzing approximation-class misspecification. Other continuous-time RL frameworks establish consistency in the vanishing-time-step limit~\cite{jia2022policya, jia2023q}, but not errors from function approximation under limited data. Beyond reinforcement learning, Koopman-based control analyses have begun to quantify finite-data errors, approximation errors, and closed-loop implications of EDMD-type surrogate models~\cite{nuske2023finite, strasser2026overview}. These results address prediction and control guarantees for finite-dimensional Koopman surrogates. Offline RL guarantees and rates, on the other hand, are mostly discrete-time. Reward-free kernel-based RL has recently been studied with nonparametric rates~\cite{vakili2024reward}, but in a well-specified value-function setting where the true value function belongs to the chosen RKHS, and with episodic exploration before reward revelation. In continuous state and action spaces, available guarantees include excess-risk bounds~\cite{farahmand2011regularization} and sup-norm error bounds~\cite{antos2007fitted}, while other guarantees are restricted to finite state or action spaces~\cite{kidambi2020morel,chen2019information}.

These results leave open the setting considered here, namely offline, continuous-time reinforcement learning with continuous state and action spaces and a nonparametric generator model. We combine nonparametric infinitesimal-generator estimation with a continuous-time HJB perturbation argument to quantify how finite-sample generator error propagates through the dynamic-programming recursion. The resulting bound separates statistical generator estimation from the RKHS approximation gap to the true optimal value function, capturing misspecification when the latter does not lie in the computational RKHS.

\section{PRELIMINARIES}\label{sec:background-and-problem}
\noindent\textbf{Notation.}
We write $[n] \coloneqq \{1,\dots,n\}$ for integers, $\odot$ for the Hadamard product. Let the state space be $\spX = \mathbb{R}^{n_s}$, the action space be $\spU \subseteq \mathbb{R}^{n_a}$, the admissible action set \(\spU_{\mathrm{ad}} \subseteq \spU\) and $\spZ\defeq \spX{\times}\spU$. The set of $k$-times continuously differentiable functions is denoted by $C^k(\IN)$. We denote by $C_0(\spX)$ the space of continuous functions on $\spX$ that vanish at infinity~\cite[Section~2]{sriperumbudur2010relation}. For a measure $\nu$ on a domain $\mathcal X$, $L_\nu^2(\mathcal X)$ denotes the space of square-integrable functions. In the offline setting, $\eta$ denotes the state-action sampling distribution on $\spZ$, and $\mu$ denotes its state marginal; accordingly, $\LOUT$ is the state-action space and $\LIN$ is the state space. We denote by $H^k(\IN)$ the Sobolev space of $k$-times weakly differentiable functions with square-integrable derivatives, and by $H_\nu^k$ the corresponding $L_\nu^2$-based Sobolev space. $\HS{\mathcal{F},\mathcal{Y}}$ denotes the space of Hilbert–Schmidt operators with norm $\|A\|_{\HS{\mathcal{F},\mathcal{Y}}}^2 \coloneqq \sum_i \|Ae_i\|_{\mathcal{Y}}^2$ for any orthonormal basis $\{e_i\}$ of $\mathcal{F}$. We write $\innerprod{\cdot, \cdot}$ for the inner product. For real Hilbert spaces $\mathcal Y$ and $\mathcal F$, $\mathcal Y\otimes\mathcal F$ denotes their Hilbert tensor product. We identify $y\otimes f$ with the rank-one operator $\mathcal F\to\mathcal Y$ that acts on an observable $h\in\mathcal F$ as $(y\otimes f)h=y\langle f,h\rangle_{\mathcal F}$. The same symbol denotes the Kronecker product in finite-dimensional block expressions. We use the standard stochastic-order notation $X_N=O_{\mathbb P}(a_N)$ when $X_N/a_N$ is bounded in probability, see~\cite[Section~2.2]{van2000asymptotic}. The symbols $\nabla_{\bx}$ and $\Delta_{\bx}$ denote the gradient and Laplacian with respect to the vector-valued variable $\bx$, and $\operatorname{Tr}$ denotes the matrix trace. For a linear operator $A$, $A^*$ denotes its adjoint, and $\operatorname{Re}z$ is the real part of a complex number $z$. Calligraphic symbols denote exact operators, while hats mark their empirical or learned counterparts. Lowercase and uppercase bold symbols denote vectors and matrices, respectively.

\subsection{CONTINUOUS-TIME MARKOV DECISION PROCESSES}
A continuous-time Markov Decision Process is described by an SDE influenced by actions $\bu_t$, on $\spX=\R^{n_s}$,
\begin{align}
     \mathrm{d}\bX_t &= \underbrace{\bigl( \bm{f}(\bX_t) + \bm{G}(\bX_t)\, \bu_t \bigr)}_{\bm d(\bX_t,\bu_t)} \, \mathrm{d}t + \bm{\Sigma}(\bX_t)\, \mathrm{d}W_t, \label{eq:rl_sde}
\end{align}
for $t \geq 0$, where $W_t$ is an $n_s$-dimensional standard Wiener process and $\bm{\Sigma}:\IN\to\mathbb R^{n_s\times n_s}$ is a matrix of diffusion coefficients. 
The action-affine formulation is ubiquitous in control and robotics, particularly for models arising from Lagrangian mechanics, thus covering a broad class of continuous-time   cyber-physical systems \cite{Nijmeijer96}. Additionally, we introduce a regularity assumption:

\begin{assumption}[Dynamics]\label{asm:dynamics}
Let $\spX=\mathbb R^{n_s}$. We assume that
$\bm{f} \in C^1(\IN)^{n_s}$ and $\bm{G} \in C^1(\IN)^{n_s \times n_a}$ are globally Lipschitz. Furthermore, $\bm{\Sigma}\in C^1(\IN)^{n_s\times n_s}$ is bounded and globally Lipschitz, and there exists $\epsilon_0>0$ such that
$\bm{\Sigma}(\bx)\bm{\Sigma}(\bx)^\top\succeq\epsilon_0 I$ for all $\bx\in\IN$.
\end{assumption}
Thus, for every measurable control signal that is essentially bounded on finite time intervals, \eqref{eq:rl_sde} admits a unique strong solution for all $t\geq0$~\cite[Theorem~5.2.1]{oksendal2003stochastic}.

\subsection{OPERATOR MODELS}
Conditional expectations of observables $\phi \in \LIN(\spX)$ are described by conditional expectation operators~\cite{kallenberg1997foundations}. Given a measurable control signal $\bu=(\bu_\tau)_{\tau\in[t,u]}$, these operators form the evolution family $\Gamma^{\bu}_{u,t}$ for $u\ge t\ge0$:
\begin{align}
    [\Gamma^{\bu}_{u,t}\phi](\bx)
    \coloneqq \expect\!\left[\phi(\bX_u)\mid \bX_t=\bx,
    (\bu_\tau)_{\tau\in[t,u]}\right].
\end{align}
Its infinitesimal generator $\mathcal{G}_t$ at time $t$ is defined on a domain with sufficient regularity $D(\mathcal{G})$~\cite{engel2000one}. For the action-affine dynamics in~\eqref{eq:rl_sde}, we use the following state--action parameterization of the generator:
\begin{equation}
  [\bFPK \phi](\bx, \bu)
  = \AbIG \phi(\bx)
  + \langle \bu,(\BbIG\phi)(\bx)\rangle,
  \label{eq:infgen_rl}
\end{equation}
which separates the autonomous generator from its state-dependent, action-linear coefficient:
\begin{subequations}\label{eq:gen-split}
\begin{flalign}
    ~&\AbIG\phi(\bx)
    {=} \nabla_{\bx}\phi(\bx)^\top\!\bm{f}(\bx)
    {+} \tfrac12\operatorname{Tr}(\bm{\Sigma}(\bx)\bm{\Sigma}(\bx)^\top\nabla_{\bx}^{2}\phi(\bx)),
    &&
    \label{eq:autonmGen}\\
    &(\BbIG\phi)(\bx)
    {=} \bm{G}(\bx)^\top\nabla_{\bx}\phi(\bx)
    \in\mathbb R^{n_a}.
    &&
    \label{eq:ctrlGen}
\end{flalign}
\end{subequations}
For a time-dependent action $\bu_t$, the corresponding generator is recovered as $[\mathcal G_t\phi](\bx)=[\mathcal G\phi](\bx,\bu_t)$.
As the time dependence in $\mathcal{G}_t$ is not from the transition dynamics but from actions $\bu_t$, we drop the subscript $t$ when in the infinitesimal setting.
\subsection{REINFORCEMENT LEARNING}
For any reward $r(\bx, \bu): \spX \times \spU \to \R$  over state-action pairs, the goal of reinforcement learning is to find the stationary policy $\bpi$ maximizing the expected discounted return -- the value function,
\begin{align}\label{eq:expected-return} 
     V(\bx)  \defeq  \expect \left[\textstyle\int_0^\infty \mathrm{e}^{-\rho t} r\left( \bX_t, \bpi(\bX_t) \right) \, \mathrm{d}t \Big| \bX_0=\bx\right],
\end{align}
with a discount exponent $\rho>0$.
To connect~\eqref{eq:expected-return} with the operator model in~\eqref{eq:infgen_rl} we introduce the partial composition operator $\mathcal{P}^\bpi: \LOUT(\spX \times \spU) \to \LIN (\spX)$. It replaces open-loop actions with the output of a policy $\bu = \bpi(\bx)$
\begin{align}
    [\mathcal{P}^\bpi r](\bx) = r(\bx, \bpi (\bx)).
\end{align}
For an admissible stationary policy $\bpi$, we write $\Gamma^\bpi_{u,t}$ for the evolution family induced by $\bu_\tau=\bpi(\bX_\tau)$, where admissibility requires the closed-loop SDE to admit a unique strong solution for all $t\geq0$.
Hence, the generator of~\eqref{eq:rl_sde} under a policy reads
\begin{equation}\label{eq:infgen_under_policy}
  [\mathcal{P}^\bpi\bFPK\phi](\bx)
  = \AbIG\phi(\bx)
  + \langle\bpi(\bx),(\BbIG\phi)(\bx)\rangle.
\end{equation}

Now, instead of sampling the SDE \eqref{eq:rl_sde} and approximating the expected return~\eqref{eq:expected-return}, we can use Fubini's theorem to model the conditional expectation directly.
\begin{align}
     V(\bx)&= \textstyle\int_0^\infty  \mathrm{e}^{-\rho t}\left[\Gamma^\bpi_{t,0}(\mathcal{P}^\bpi r)\right](\bx)\,\mathrm{d}t\label{eq:expected-return2} \\
    &=  \big[\big( \rho \mathcal{I} -\mathcal{P}^\bpi\mathcal{G} \big)^{-1} (\mathcal{P}^\bpi r)
    \big](\bx)  \label{eq:ResolventValFun}
\end{align}
The inverse in \eqref{eq:ResolventValFun} is exactly the definition of the \textit{resolvent operator}~\cite[Theorem~1.10]{engel2000one} applied to the reward function under any policy $\bpi$ for which the SDE~\eqref{eq:rl_sde} admits a well-defined solution.
As the resolvent \textit{linearly maps} the reward to the value function, this establishes the infinitesimal generator as a fundamental object in continuous-time Markov decision processes.

\subsection{HAMILTON-JACOBI-BELLMAN EQUATION}
Before connecting to optimal control theory, we recognize that the identity in~\eqref{eq:ResolventValFun} characterizes $V$ for any \emph{fixed} $\bpi$. To obtain the \textit{optimal} value function, we optimize over policies. The resulting optimal value function $\Vstar$ must satisfy the stationary discounted HJB
\begin{equation}
  \max_{{
    \|\bpi\|_{L^\infty} < \infty}}\Big\{[\mathcal{P}^\bpi\mathcal{G}\Vstar](\bx) + r(\bx,\bpi(\bx))\Big\}=\rho\Vstar(\bx)
  ,
  \label{eq:discountHJB1}
\end{equation}
expressed in terms of the generator and the partial composition operator. Precise conditions under which the HJB \eqref{eq:discountHJB1} has a well-defined solution $\Vstar$ can be found in \cite{houska2025convex, fleming2006controlled}.

\subsection{PROBLEM STATEMENT} 
Given the recorded data of a continuous-time system in infinitesimal form
\begin{align}
   \bigl\{\, \bm d^{(i)},\, (\bm{s}^{(i)}, \bm{a}^{(i)})\,\bigr\}_{i\in[N]},
    \label{eq:dataset} 
\end{align}
Our objective is to find the optimal value function
\begin{equation}
    \Vstar \defeq  \max_{{
    \|\bpi\|_{L^\infty} < \infty}} \big( \rho \mathcal{I} -\mathcal{P}^\bpi \mathcal{G} \big)^{-1} \mathcal{P}^\bpi r,
    \label{eq:optValFun}
\end{equation}
which maximizes the expected cumulative discounted rewards in \eqref{eq:expected-return}. Here $\bm d^{(i)}$ is the $i$th measurement of the drift, $\bm d^{(i)}=\bm d(\bm s^{(i)},\bm a^{(i)})=\bm f(\bm s^{(i)})+\bm G(\bm s^{(i)})\bm a^{(i)}$. The sampled diffusion covariances $\bm\Sigma(\bm s^{(i)})\bm\Sigma(\bm s^{(i)})^\top$ may be known from physics or simulation, estimated separately, or approximated by an isotropic upper-bound surrogate $2\epsilon I$, i.e., $\bm\Sigma=\sqrt{2\epsilon}\,I$. If the drift is not measured directly, proxy targets can be constructed from state data using finite differences or interpolation-based methods \cite{chartrand2011numerical}.
Although standard offline RL datasets often provide only discrete transitions, drift measurements are available in many continuous-time data sources. Physics-based and scientific simulators frequently expose velocities, accelerations, forces, or model quantities from which the drift can be read off or reconstructed, making \eqref{eq:dataset} a natural representation for such systems.

We aim to derive an algorithm $\textsf{O-CTRL}$ based on the following properties:
\begin{enumerate}[label={$\mathbf{(P1)}$},leftmargin=7ex,nolistsep] \item\label{eq:jointUpdate} \textbf{World Model:} By learning an approximation of the infinitesimal generator $\mathcal{G}$, the knowledge of the \textit{reward is required only at inference}.\end{enumerate} 
\begin{enumerate}[label={$\mathbf{(P2)}$},leftmargin=7ex,nolistsep] \item\label{eq:Cvg} \textbf{Formal Guarantees:} Provide explicit, \textit{interpretable error bounds} on the algorithm’s output. \end{enumerate} \begin{enumerate}[label={$\mathbf{(P3)}$},leftmargin=7ex,nolistsep] \item\label{eq:Ope} \textbf{Simple Algorithm:} Optimization is carried out through recursive dynamic programming updates, \textit{implemented by a single for-loop} (scan).

\end{enumerate}

In a nutshell, our goal is an end-to-end pipeline: we start from dynamical system data to construct a value-function approximation with guarantees \ref{eq:Cvg}. These are obtained by relating operator world model \ref{eq:jointUpdate} learning errors to value-function errors. Building on RKHS-based operator models, we formulate a dynamic-programming recursion that jointly updates value and policy and consists only of a single for-loop scan \ref{eq:Ope}. The introduced building blocks are key in making our approach modular: reward shaping or task changes become plug-and-play without retraining the world model.

\section{THE OPTIMAL CONTROL PERSPECTIVE}\label{sec:OfflineRLMeetsOptimalControl}
This section reviews established stochastic optimal control concepts in the operator-theoretic form used by $\textsf{O-CTRL}$.
Although fixed-policy evaluation is linear in the reward through Eq.~\eqref{eq:ResolventValFun}, policy optimization in RL is generally nonlinear and challenging because the policy-to-value map $\bpi\mapsto V$ is nonlinear.
In contrast, under control-affine dynamics and a strongly convex action penalty, the Hamiltonian maximization over a convex admissible action set admits a unique global maximizer characterized by the Fenchel conjugate. When this conjugate is explicitly available, it yields a closed-form policy update offering simplicity and tractability analogous to LQR.
Moreover, because continuous-time Q-functions are ill-defined as the discretization vanishes~\cite{tallec2019making}, our operator-theoretic HJB formulation avoids dependence on an arbitrary time increment or surrogate Q-function~\cite{doya2000reinforcement,tallec2019making,jia2023q}.

\subsection{OPTIMAL CONTROL FOR POLICY OPTIMIZATION}
Under the uniform nondegeneracy condition in \Cref{asm:dynamics}, we work with the classical notion of weak solutions in Sobolev spaces~\cite[Chapter~IV, Sections~4--5]{fleming2006controlled}, \cite{houska2025convex}.

To compute the steady state $\Vstar$ solving~\eqref{eq:discountHJB1}, we evolve the value-iteration flow for $t \geq 0$ on $\mathbb{S}$, where $m\ge 1$,
\begin{align}
  \!\!\!\!\!  \dot{V}(t,\bx)=\mathcal{T}\big(V(t,\cdot)\big)(\bx), 
     \; V(0,\bx)=V_0(\bx)\in H_\mu^m(\spX),
    \label{eq:hjb_evolution} 
\end{align}
under standard regularity assumptions, where $\mathcal T$ denotes the HJB operator defined in~\eqref{eq:HJB}. To pave the way for a simple dynamic-programming recursion and to enable a joint value–policy update \ref{eq:Ope}, we require a closed-form expression for the operator $\mathcal T:H_\mu^m(\spX)\to\LIN(\spX)$. For this purpose, and to guarantee that $\mathcal T$ is well-defined, we impose the reward structure stated in the following Assumption. This choice ensures that the maximization in \eqref{eq:discountHJB1} can be expressed as a Fenchel-conjugate term \cite{houska2025convex}.

\begin{assumption}[Reward]\label{asm:reward}
We model the reward as a function
\(r(\bx,\bu)=r_\bx(\bx)-c_\bu(\bx,\bu)\),
which is continuously differentiable in both \(\bx\) and \(\bu\). We require strong-convexity of the action penalty $c_\bu(\bx, \bu)$ w.r.t. actions. The state reward \(r_\bx(\bx)\) is either known/defined a priori
or unknown but provided in the dataset~\eqref{eq:dataset}.
\end{assumption}
This reward structure is ubiquitous, with quadratic or smoothed $p$-norm penalties on control effort \cite{anderson2007optimal} and separable state terms \cite{doya2000reinforcement}. It is widely used in physics-based benchmarks \cite{towers2024gymnasium} and modern RL tasks such as locomotion \cite{hwangbo2019learning}.

Using Assumption ~\ref{asm:reward}, we get an explicit \textit{policy update rule} from the Fenchel conjugate of the action penalty.\footnote{Also commonly written as $c_{\bu}^*(\bl;\bx)$ in the literature.} It reads
\begin{align}\label{eq:PolicyUpdate}
    \mathcal{D}_\bu(\bl,\bx)  \coloneqq \max_{\bu \in \spU_{\mathrm{ad}}} \{\langle \bu, \bl \rangle- c_{\bu}(\bx, \bu)\},
\end{align}
which is well-defined and admits a unique maximizer $\bu^\star(\bl,\bx) = \nabla_{\bl} \mathcal{D}_\bu(\bl,\bx)$ \cite{boyd2004convex}. 

\begin{proposition}[Fenchel Lipschitzness under bounded actions]
\label{prop:main:LipschitzFenchel}
If the admissible action set $\spU_{\mathrm{ad}}$ is compact and convex, then the restricted Fenchel conjugate $\mathcal D_{\bu}(\cdot,\bx)$ is globally Lipschitz continuous in its first argument, uniformly in $\bx$, with $L_{\mathcal D}\le \sup_{\bu\in\spU_{\mathrm{ad}}}\|\bu\|$.
\end{proposition}
\begin{IEEEproof}
Given action constraints defined by the compact convex set $\spU_{\mathrm{ad}}$, the maximizer of the Fenchel conjugate satisfies $\nabla_{\bl} \mathcal D_\bu(\bl,\bx)\in\spU_{\mathrm{ad}}$. Thus $\|\nabla_{\bl} \mathcal D_\bu(\bl,\bx)\|\le \sup_{\bu\in\spU_{\mathrm{ad}}}\|\bu\|$ for all $\bl$ and $\bx$, and Lipschitzness of $\mathcal D_\bu(\cdot,\bx)$ follows.
\end{IEEEproof}

\begin{remark}[Smooth approximations]
If $c_\bu$ or the action constraints are nonsmooth, smooth approximations of $\mathcal D_{\bu}(\cdot,\bx)$ can be obtained, e.g., by Nesterov smoothing~\cite{nesterov2005smooth}.
\end{remark}

\begin{remark}[Explicit Fenchel conjugate]
For quadratic penalties $c_\bu(\bx,\bu)=\frac{1}{2}\bu^\top R(\bx)\bu$ and positive definite $R(\bx)$, the unconstrained Fenchel maximizer is $\bu^\star(\bl,\bx)=R(\bx)^{-1}\bl$. Under compact box constraints and diagonal $R(\bx)$, the restricted maximizer is the corresponding saturated feedback.
 
\end{remark}

After isolating the state reward from the action maximization and substituting the Fenchel conjugate $\mathcal{D}_{\bu}$ in \eqref{eq:discountHJB1},  we obtain the operator/Fenchel representation of the classical discounted HJB
\begin{align}\label{eq:HJB}
    \mathcal T(V)= -\big(\rho \mathcal{I} - \mathcal A\big)V + r_{\bx} + \mathcal D_{\bu}\!\big(\mathcal B V,\cdot\big)
\end{align}
for $V \in H_\mu^m(\spX)$, where we also leveraged the action-affinity of the dynamics \eqref{eq:rl_sde}.
The final term is understood pointwise in the state: $[\mathcal D_{\bu}(\mathcal BV,\cdot)](\bx)\coloneqq\mathcal D_{\bu}((\mathcal BV)(\bx),\bx)$. Thus $\mathcal BV$ is evaluated only at $\bx$, while $\bu$, the optimization variable, is eliminated by the Fenchel maximum in~\eqref{eq:PolicyUpdate}.
As the maximum in \eqref{eq:PolicyUpdate} is attained at $\bu^\star(\bl,\bx)$, we find a, perhaps unsurprising, structure 
\begin{align}\label{eq:RLunpack}
    \mathcal{D}_\bu(\bl,\bx)  = \underbrace{\langle \bu^\star(\bl,\bx), \bl \rangle}_{{\substack{\text{improvement w/} \\ \text{optimal policy}}}}- \underbrace{c_{\bu}(\bx, \bu^\star(\bl,\bx))}_{\text{regularizer/penalty}},
\end{align}
where the \textit{costate} $\bm{\lambda} \defeq \mathcal{B} V$ is the continuous-time analogue of the advantage signal in discrete-time RL. The first term in \eqref{eq:RLunpack} describes the improvement under the optimal policy, a \textit{deterministic, continuous-time, analogue} to expected advantage. The second term, on the other hand, serves as the regularization and constraint-enforcing term for the actions.
The decomposition in \eqref{eq:RLunpack} can therefore be viewed as an exact regularized greedy step in costate space, yielding a smooth, unique, and Lipschitz policy-improvement map under the stated bounded-action conditions.

\section{STATISTICAL LEARNING OF WORLD MODELS}
\label{sec:operator-models}
To build an infinitesimal world model and solve the HJB via a simple dynamic programming recursion \ref{eq:Ope}, we seek a data-driven approximation of the infinitesimal generator $\ctrlbFPK: D(\ctrlbFPK) \to  \LOUT$, with $D(\ctrlbFPK)$ a space with sufficient Sobolev regularity. This, in turn, allows for obtaining data-based surrogates of the operators $\mathcal{A}$ and $\mathcal{B}$ \eqref{eq:autonmGen}-\eqref{eq:ctrlGen}. To ensure computational tractability, we restrict the approximating value function and policy to reproducing kernel Hilbert spaces, enabling tractable computation within a rich (infinite-dimensional) parameterization.
\subsection{REPRODUCING KERNEL HILBERT SPACES}

We consider RKHSs $\spIN$/$\spOUT$ that are a subset of sufficiently regular state functions/$\LOUT$-integrable functions \cite[Chapter 4.3]{IngoSteinwart2008SupportMachines} with associated canonical feature maps $\featx: \IN \rightarrow \spIN$ and $\featz: \OUT \rightarrow \spOUT$. To perform the differential calculus required for infinitesimal generators, we consider $k: \IN \times \IN \rightarrow \Set{R}$ to be a symmetric and positive definite kernel function such that $k\in C^{2,2}(\spX\times\spX)$ and $\spIN$ the corresponding RKHS \cite{IngoSteinwart2008SupportMachines}, with norm denoted as $\|\cdot\|_{\spIN}=\sqrt{\langle \cdot,\cdot\rangle_{\spIN}}$. Moreover, for all $\bx, \bx' \in \spX$, we have that $k(\bx, \bx') = \innerprod{\featx(\bx),\featx({\bx'})}_{\spIN}=\innerprod{k(\cdot,\bx),k(\cdot,{\bx'})}_{\spIN}$ and the reproducing property $h(\bx) = \innerprod{h,k(\cdot,\bx)}_{\spIN}$ holds for all $\bx \in \IN$ and all observables $h \in \spIN$. We further assume that we are working with universal kernels and that $k(\bx, \bx') < \infty$.

The action-affine dynamics in \eqref{eq:rl_sde} induce an action-affine output RKHS \cite{bevanda2025nonparametric}. A convenient construction is the tensor product space $\spOUT \coloneqq \mathbb R^{1+n_a} \otimes \spIN$. For $\bm z=(\bx,\bm a)$, the corresponding reproducing kernel can be written as ${k}(\bm{z},\bm{z}^\prime) = {k}(\bx,\bx^\prime)(1{+}\innerprod{\bm{a}, \bm{a}^\prime})$. This choice preserves the split between the autonomous operator $\mathcal A$ and the controlled operator $\mathcal B$, while still allowing both parts to be learned through one regularized operator-regression problem.

\subsection{POPULATION RISK}
We look for an RKHS approximation $G: \spIN \to \spOUT$ of the generator acting on RKHS observables.
At the population level, recall that $\eta$ denotes the offline state-action sampling distribution, with state marginal $\mu$. To make the RKHS-level target and the Hilbert-Schmidt risk well defined, we use the following compatibility condition.
\begin{assumption}[Generator-RKHS compatibility]\label{asm:generator-rkhs}
There exist $m\ge 1$ and $y>m+n_s/2$ such that $\spIN$ embeds continuously into $H_\mu^y(\spX)$, the sampling distribution has finite second action moment, and the controlled generator admits a continuous realization $\ctrlbFPK:H_\mu^m(\spX)\to\LOUT$.
\end{assumption}
By Maurin's theorem~\cite{adams2003sobolev}, the Sobolev embedding $H_\mu^y(\spX)\hookrightarrow H_\mu^m(\spX)$ is Hilbert-Schmidt.\footnote{For the strong form in \eqref{eq:infgen_rl} one may take $m=2$, while weak formulations may allow $m=1$. Gaussian RKHSs satisfy the finite-order Sobolev embedding requirement. Matérn-$\nu$ kernels have native space $H^{\nu+n_s/2}$ on regular bounded domains~\cite{IngoSteinwart2008SupportMachines}, so Matérn-$5/2$ covers the strong-form choice $m=2$. Thus Matérn-$3/2$ is excluded from this strong-form construction. However, one may choose a smoother kernel or develop a weak-form estimator with lower derivative requirements.} Since composing a Hilbert-Schmidt operator with bounded linear operators
again yields a Hilbert-Schmidt operator, the composed inclusion $\mathsf{E}^\prime:\spIN\to H_\mu^m(\spX)$ is Hilbert-Schmidt, and therefore
\[
\target \defeq \ctrlbFPK \mathsf{E}^\prime:\spIN\to\LOUT
\]
belongs to $\mathrm{HS}(\spIN,\LOUT)$.
The output RKHS is embedded into the ambient state-action space by $\injectOUT:\spOUT\hookrightarrow\LOUT$. The learned operator is obtained from the Hilbert-Schmidt risk
\begin{align}
    \mathcal R(\mathsf{G})
    \defeq
    \|\target-\injectOUT \mathsf{G}\|^2_{\mathrm{HS}(\spIN,\LOUT)},
\label{eq:main:operator_risk}
\end{align}
where $\mathsf{G}\in\mathrm{HS}(\spIN,\spOUT)$. This risk separates the modeling question from the statistical one through the decomposition
\begin{flalign}
&\mathcal R(\mathsf{G}) \!
= \! \underbrace{\|[\mathsf{Id}\!-\!\mathsf{P}_{\spOUT}]\target \! \|^2_{\mathrm{HS}}}_{\mathclap{\mathsfit B(\spOUT)^2\ (\mathrm{representation\ bias})}}
+
\underbrace{\|\mathsf{P}_{\spOUT}\target\!-\injectOUT \mathsf{G}\|^2_{\mathrm{HS}}}_{\mathclap{\mathrm{projected\ estimation}}}.
&& \label{eq:main:risk_decomp}
\end{flalign}
where $\mathsf{P}_{\spOUT}$ is the $\LOUT$ projection onto the closure of the output RKHS. For universal kernels (for instance, Gaussian radial basis kernels~\cite[Chapter~4]{IngoSteinwart2008SupportMachines}), the representation term can be made arbitrarily small, so the main finite-sample analysis focuses on estimating the projected generator. Let $\mathrm d\featx(\bm z)\in\spIN$ denote the generator representer satisfying $\innerprod{\mathrm d\featx(\bm z),h}_{\spIN}=[\target h](\bm z)$ for $h\in\spIN$. The associated population covariance and cross-covariance operators are
\begin{align}
\mathsf{C}_{\mathsf{z}\mathsf{z}}
\defeq \expect_{\eta}[\featz(\bm z)\otimes\featz(\bm z)],\quad
\mathsf{C}_{\mathsf{z}\mathrm d}
\defeq \expect_{\eta}[\featz(\bm z)\otimes \mathrm d\featx(\bm z)].
\notag
\end{align}
To make the typically ill-posed inverse problem stable, we use Tikhonov regularization. In the RKHS regression setting, this is the population form of Kernel Ridge Regression (KRR):
\begin{align}
{\mathsf{G}}_\reg
=
\argmin_{\mathsf{G}\in\mathrm{HS}(\spIN,\spOUT)}
\mathcal R(\mathsf{G})+\gamma\|\mathsf{G}\|^2_{\mathrm{HS}}
=
(\mathsf{C}_{\mathsf{z}\mathsf{z}}+\gamma I)^{-1}\mathsf{C}_{\mathsf{z}\mathrm d}.
\label{eq:main:population_krr}
\end{align}

\subsection{EMPIRICAL RISK MINIMIZATION}
Population-level quantities in \eqref{eq:main:population_krr} are typically unavailable, so we approximate them using the $N$ recorded samples.
The regression target is not the vector field itself, but the infinitesimal action of the controlled generator on RKHS observables. For sample $i$, this action has the RKHS representer

\begin{align}
\mathrm d\featx(\bx^{(i)};\bm d^{(i)})
\defeq &
\sum_{k=1}^{n_s}
d_k^{(i)}\,\partial_{s_k}\featx(\bxi) \notag \\
&+
\frac{1}{2}
\sum_{k,\ell=1}^{n_s}
\left[
\bm{\Sigma}(\bxi)\bm{\Sigma}(\bxi)^\top
\right]_{k\ell}
\partial^2_{s_k s_\ell}\featx(\bxi).
\notag
\end{align}

For the drift measurements $\bm d^{(i)}$ defined in the problem statement, with $\bm z^{(i)}=(\bx^{(i)},\bu^{(i)})$, this empirical representer is the sample realization of $\mathrm d\featx(\bm z^{(i)})$. By the derivative reproducing property, for any observable $h\in\spIN$,
$
\innerprod{\mathrm d\featx(\bx^{(i)};\bm d^{(i)}),h}_{\spIN}
=
[\target h](\bx^{(i)},\bu^{(i)}).$

The empirical KRR estimator used by \textsf{O-CTRL} is obtained by replacing the population covariance operators in \eqref{eq:main:population_krr} with the sampling operators
\begin{align}
\ES_{\mathsf{z}} &: \spOUT \to \mathbb{R}^N, &
(\ES_{\mathsf{z}} g)_i &\defeq \innerprod{g,\featz((\bxi, \bui))}_{\spOUT}, \notag\\
\ES_{\mathsf{s}} &: \spIN \to \mathbb{R}^N, &
(\ES_{\mathsf{s}} h)_i &\defeq \innerprod{h,\featx(\bx^{(i)})}_{\spIN}, \notag \\
\ES_{\mathrm d} &: \spIN \to \mathbb{R}^N, &
(\ES_{\mathrm d}h)_i &\defeq \innerprod{\mathrm d\featx\!\bigl(\bx^{(i)};\bm d^{(i)}\bigr),h}_{\spIN}, \notag \\
\widehat{\mathsf{U}} &: \mathbb{R}^{n_a} \to \mathbb{R}^N, &
(\widehat{\mathsf{U}} \bm a)_i &\defeq \langle \bm a^{(i)}, \bm a\rangle. \notag
\end{align}
Equivalently, it minimizes the empirical squared loss
\begin{align}
        \widehat{\mathcal{R}}(\mathsf{G}) = \frac{1}{N}\norm{\ES_{\mathsf{d}} - \ES_{\mathsf{z}} \mathsf{G}}^2_{\mathrm{HS}(\spIN,\Set{R}^N)}. \label{eq:EmpRisk}
\end{align}
With the same Tikhonov penalty, the empirical estimator is
	 \begin{align}\label{eq:regfFPKonHestim}
\!\!\!\!\EEstim_{\reg}{\defeq}
\argmin_{\mathsf{G} \in  \mathrm{HS} (\spIN, \spOUT)}
\widehat{\mathcal{R}}(\mathsf{G}){+} \gamma\| \mathsf{G}\|^2_{\mathrm{HS}}
= \aES_\mathsf{z}\bm{K}^{-1}_{\gamma}{\ES}_{\mathsf{d}}
= \concat{\EAEstim_{{\reg}}}{\EBEstim_{{\reg}}},
\end{align}  
where $\bm{K}_{\gamma} = \ES_\mathsf{z}\aES_\mathsf{z}+N\gamma \bm{I}$ denotes the regularized state-action Gram matrix, with block components $\EAEstim_\reg\in\mathrm{HS}(\spIN,\spIN)$ and $\EBEstim_\reg\in\mathrm{HS}(\spIN,\mathbb R^{n_a}\otimes\spIN)$. \Cref{eq:regfFPKonHestim} is the empirical KRR approximation of $\mathcal{G}$ over $\spIN \to \spOUT$.

The last equality makes the action-affine structure explicit: since
$\featz(\bx,\bu)=(1,\bu)\otimes\featx(\bx)$, the KRR solution decomposes
into an autonomous block and a controlled block. For $w\in\mathbb R^{n_a}\otimes\spIN$, there is a unique
$w(\bx)\in\mathbb R^{n_a}$ with
$\langle w,\bu\otimes\featx(\bx)\rangle_{\mathbb R^{n_a}\otimes\spIN}=\langle\bu,w(\bx)\rangle$ for every $\bu\in\mathbb R^{n_a}$, which under
the canonical isometric isomorphism $\mathbb R^{n_a}\otimes\spIN\cong(\spIN)^{n_a}$
is componentwise evaluation by the reproducing property. Hence, the controlled
block has the scalar state--action realization
$[\injectOUT\concat{0}{w}](\bx,\bu)=\langle\bu,w(\bx)\rangle$. Then, for any observable $h\in\spIN$,
\begin{align}
[\widehat{\mathsf{G}}_\reg h](\bx,\bu)
=
\innerprod{\EAEstim_\reg h,\featx(\bx)}_{\spIN}
+
\innerprod{\EBEstim_\reg h,\bu\otimes\featx(\bx)}_{\mathbb R^{n_a}\otimes\spIN}.
\notag
\end{align}
Thus $\EAEstim_\reg$ estimates the autonomous generator $\mathcal A$, while $\EBEstim_\reg$ estimates the vector-valued coefficient operator $\mathcal B$. The stack above is a block operator with one autonomous component in $\spIN$ and $n_a$ controlled components in $(\spIN)^{n_a}$.

\begin{algorithm*}[t]
\caption{\textsc{\textbf{O}perator-Based \textbf{C}ontinuous-\textbf{T}ime {o}ffline \textbf{R}einforcement \textbf{L}earning ($\textsf{O-CTRL}$)}}
\label{alg:O2RL}
\newcommand{\ip}[2]{\langle #1,#2\rangle}
\newcommand{\KS}{\bm{K}_{\mathsf{a}}}
\newcommand{\KA}{\bm{K}_{\mathsf{a}}}
\newcommand{\kS}{k_{\mathsf{s}}}
\newcommand{\kA}{k_{\mathsf{a}}}
\newcommand{\Reg}{\gamma}
\newcommand{\Tr}{\operatorname{Tr}}

\begin{algorithmic}
\vspace{0.25em}
{
\Require $\{\bm d^{(i)},(\bx^{(i)},\bu^{(i)})\}_{i=1}^{N}$; $\{r_{\bx}(\bx^{(i)})\}_{i=1}^{N}$;\;
state kernel $\kS(\bx,\bx')=\ip{\featx(\bx)}{\featx(\bx')}$ with $\featx:\mathbb{S}\!\to\!\spIN$;\;
regularization~$\Reg>0$; diffusion covariances $\{\bm{\Sigma}(\bxi)\bm{\Sigma}(\bxi)^\top\}_{i=1}^N$; discount ~$\rho>0$; timestep $\Delta t > 0$; $\textsc{tol}>0$; $k_{\max}\!\in\!\mathbb{N}$; dual $\bu^\star(\cdot,\cdot)$.
}

\begin{algoblock}{\textsc{representation}}
  \Statex \textbf{set:} $(\bm{k}_{\mathsf{s}}(\bx))_i:=\kS(\bx,\bxi)$,\; $(\Kx)_{ij}:=\kS(\bx^{(i)},\bx^{(j)})$, \; $\bm{U} := [\bu^{(1)},\dots,\bu^{(N)}]^\top$
\end{algoblock}
\begin{algoblock}{\textsc{world model}}
  \Statex \textbf{compute:} $\bm{K}_\Reg := \Kx+\Kx\odot\bm{U} \bm{U} ^\top+N\Reg\,\bm{I}$ \Comment{Gram}
  \Statex \textbf{compute:} $(\Kyx)_{ij}:=\innerprod{\bm d^{(i)}, \nabla_{\bxi} k(\bxi, \bxj)} +
\frac{1}{2}
\operatorname{Tr}\!\left(
\bm{\Sigma}(\bxi)\bm{\Sigma}(\bxi)^\top
\nabla_{\bxi}^{2}k(\bxi,\bxj)
\right)$ \Comment{target}
  \Statex \textbf{compute:} $\SFA:=\bm{K}_\Reg^{-1}\Kyx$,\quad $\SFB:=[\mathrm{diag}(\bm{U} \bm{e}_{i})\SFA]_{i \in [n_a]}$ \Comment{dynamics}
\end{algoblock}

\begin{algoblock}{\textsc{dynamic programming}}
  \Statex \textbf{compute:} $\bm{r}:=\bm{K}_\Reg^{-1}\bm{y}_r$, where $\bm{y}_r:=[r_\bx(\bx^{(1)}),\ldots,r_\bx(\bx^{(N)})]^\top$ \Comment{reward}
  \Statex \textbf{set:} $\bm{\lambda}(\bx):=\ip{\bm{I}_{n_a}\otimes\bm{k}_{\mathsf{s}}(\bx)}{\bm{\lambda}}, \quad \bm{M}\defeq(\bm{I}+\Delta t(\rho\,\bm{I}-\SFA))^{-1}$ \Comment{costate \& IMEX propagator}
  \Statex \textbf{set:} $\bm{D}_\bu(\bl):=\bm{K}_\Reg^{-1}\!\big[\ip{\bu^\star(\bl(\bxi),\bxi)}{\bl(\bxi)}-c_{\bu}(\bxi,\bu^\star(\bl(\bxi),\bxi))\big]_{i=1}^{N}$ \Comment{Fenchel dual}
  \State  $\bm{v}^{(0)}\gets\bm{0}$,\; $k\gets0$
  \Repeat
    \State $\bm{v}^{(k+1)}\gets \bm{M}\big[\bm{v}^{(k)}+\Delta t\big(\bm{r}+\bm{D}_\bu(\SFB\bm{v}^{(k)})\big)\big]$ \Comment{implicit-explicit step}
    \State $k\gets k+1$
  \Until{$\|\bm{v}^{(k)}-\bm{v}^{(k-1)}\|\le\textsc{tol}$ \textbf{ or } $k=k_{\max}$}
\end{algoblock}

\Ensure $\widehat{V}_k(\bx)=\ip{\bm{v}^{(k)}}{\bm{k}_{\mathsf{s}}(\bx)}$ \; and \;
$\widehat{\bm{\pi}}_k(\bx):=\bu^\star\!\big(\ip{\bm{I}_{n_a}\otimes\bm{k}_{\mathsf{s}}(\bx)}{\SFB\,\bm{v}^{(k)}},\bx\big)$ \Comment{value \& policy}
\end{algorithmic}
\end{algorithm*}

\section{OPERATOR-BASED DYNAMIC PROGRAMMING}\label{sec:operator-dynamic-programming}

Ultimately, we construct a tractable approximation to \eqref{eq:hjb_evolution}, using the world model \ref{eq:jointUpdate}, derived from \eqref{eq:regfFPKonHestim}, to obtain a value function estimate $\widehat{V} \in \spIN$ via 
\begin{align}
    \widehat{\mathcal T}(\EV)=-(\rho \mathsf{Id}-\EAEstim_{\reg})\EV+\widehat{r}_\bx+\widehat{\mathcal D}_{\bu}(\EBEstim_{\reg}\EV,\cdot),
    \label{eq:approxHJB}
\end{align}
where $\widehat{\mathcal{T}}$ can be interpreted as an approximation of the HJB operator $\mathcal{T}$ in \eqref{eq:HJB}, with the data-based approximations $\widehat{r}_\bx$ and $ \widehat{\mathcal{D}}_\bu$ defined in the proposition below.

\begin{restatable}{proposition}{approxHJBflow}
\label{prop:approxHJBflow}
Let $\widehat{r}_\bx=\ES_{\mathsf{s}}^{*}\,\bm{r}, \widehat{\mathcal{D}}_\bu(\bl,\cdot) = \ES_{\mathsf{s}}^{*}\bm{D}_\bu(\bl)$, where $\ES_{\mathsf{s}}^{*}:\mathbb{R}^{N}\to\spIN$ is the adjoint of the sampling operator $\widehat{\mathsf{E}}_{\mathsf{s}}$. Let the transition dynamics be described by \eqref{eq:regfFPKonHestim}. Then the flow of  \eqref{eq:approxHJB} resides in a finite-dimensional subspace
\begin{align}
\big\langle \dot{\bv}(t,\cdot),\,\bm{k}_{\mathsf s}(\bx)\big\rangle
=
\big\langle \bm{T}\!\big(\bv(t,\cdot)\big),\,\bm{k}_{\mathsf s}(\bx)\big\rangle,
\label{eq:finite_hjb_evol}
\end{align}
with coefficient vector $\bv(t)\in\mathbb R^N$ and
\begin{align}
\bm{T}(\bv) \defeq -(\rho\bm{I}-\SFA)\,\bv + \big(\bm{r} + \bm{D}_{\bu}\!(\SFB\,\bv)\big),
\end{align}
where $\ES_{\mathsf{s}}\featx(\bx)=\bm{k}_{\mathsf s}(\bx)$ is the sampled canonical map, and $$\SFA\defeq \bm{K}_{\gamma}^{-1}\,{\ES}_{\mathsf d}\,\aES_\mathsf{s},  \quad \SFB\defeq\big[\mathrm{diag}(\widehat{U}\,\bm{e}_k)\,\SFA\big]_{k\in[n_a]}.$$ 
The reward and Fenchel-dual terms in the RKHS are
$\bm{r}=\bm{K}_{\gamma}^{-1}\,[\,r_{\bx}(\bx^{(1)}),\dots,r_{\bx}(\bx^{(N)})\,]^\top$ and
$\bm{D}_{\bu}(\bl)=\bm{K}_{\gamma}^{-1}\, \big[ \mathcal{D}_\bu(\bl(\bxi),\bxi)\big]_{i \in [N]}$, respectively.
\end{restatable}

This proposition follows by substituting the estimators \eqref{eq:regfFPKonHestim} and the representation $\EV(t,\cdot)=\widehat{\mathsf{E}}_{\mathsf{s}}^{*}\bv(t,\cdot)$ into~\eqref{eq:approxHJB} and testing against $\featx(\bx)$ to obtain \eqref{eq:finite_hjb_evol}. The finite-dimensional derivation is given in Appendix~\ref{app:sec:octrl-derivations}. $\SFA$ approximates the autonomous generator block. The compact notation $\SFB=[\mathrm{diag}(\bm U\bm e_i)\SFA]_{i\in[n_a]}$ in \Cref{alg:O2RL} denotes a vertical concatenation, so $\SFB\in\mathbb{R}^{N n_a\times N}$ stacks the controlled generator blocks. Consequently, $\SFB\bv$ collects the action-wise coefficients of the controlled generator applied to a candidate value vector, whose pointwise evaluations are used as the costate arguments in the Fenchel-dual policy-improvement term.

The last step toward obtaining a tractable dynamic-programming recursion~\ref{eq:Ope} is to discretize \eqref{eq:finite_hjb_evol} in time. To this end, we treat the stiff linear part $(\rho\bm I-\SFA)$ \textit{implicitly} to ensure unconditional numerical stability of the linear part \cite{he2013euler}, while the nonlinear term $\bm D_{\bu}(\SFB\,\bv)$ is worked out \textit{explicitly} to avoid costly nonlinear solves.

\begin{restatable}{corollary}{IMEX}
\label{cor:IMEX}
Let the step-size $\Delta t > 0$. Then the first-order Implicit-Explicit (IMEX) scheme~\cite{he2013euler} reads
    \begin{align}
\bv^{(k+1)}
&= \bm{M}\left[\bv^{(k)}
+ \Delta t\big(\bm{r} + \bm D_{\bu}(\SFB \bv^{(k)})\big)\right],
    \end{align}
    with $\bm{M}= (\bm{I} + \Delta t (\rho \bm{I} - \SFA ))^{-1}$  naturally following from the implicit discretization.
\end{restatable}
This yields the operator-theoretic recursion used in \Cref{alg:O2RL}. The discretization derivation is given in Appendix~\ref{app:sec:octrl-derivations}.
\begin{remark}[Time-step flexibility]
In Algorithm~1, $\Delta t$ is only a numerical step for approximating the continuous-time HJB flow and can be changed without relearning the infinitesimal generator. The learned generator can be reused to derive a discrete-time controller for any chosen actuator sampling period by discretizing the discount and costate recursion accordingly, allowing sensor and actuator sampling rates to differ. In contrast, a discrete-time model is learned for a fixed sampling period, and its policy normally updates the control at that same rate.

\end{remark}

\section{END-TO-END FINITE-SAMPLE GUARANTEES}\label{sec:end-to-end-bounds}
In this section, we aim to bound $\|V^\star-\widehat V_{k}\|_{\LIN}$, 
the $L^2$ error under the state marginal $\mu$, since value functions are state-only. Here, $V^\star$ is the optimal value function \eqref{eq:discountHJB1} and $\EV_{k}$, the output of \Cref{alg:O2RL}, is a numerical approximation
(e.g., IMEX, \Cref{cor:IMEX}) of $\EV_T$. The latter is obtained by evolving the approximated HJB flow using $\widehat{\mathcal{T}}$  \eqref{eq:approxHJB}
for a large time $T>0$. To capture the \emph{end-to-end} pipeline
from data to value function approximation, we write $\EVstar$ for a stationary point of $\widehat{\mathcal T}$.
\begin{figure}[t!]
    \centering
    \resizebox{\columnwidth}{!}{\begin{tikzpicture}[
  every node/.style={font=\footnotesize},
  symbol/.style={font=\normalsize},
  >={Stealth},
  node distance=2cm and 2.4cm, 
  auto
]
  


  \node[symbol] (EVinf) {$\EV^\star$};
  \node[symbol, right=2.4cm of EVinf] (Vinf) {$\Vstar$};
  \node[symbol, right=2.4cm of Vinf] (G) {$\mathcal{G}$};

  \node[symbol, below=1.9cm of EVinf.base, anchor=base] (EVT) {$\EV_T$};
  \node[symbol, below=1.9cm of Vinf.base, anchor=base] (EVdtT) {$\EV_{k}$};
  \node[symbol, below=1.9cm of G.base, anchor=base] (GN) {$\widehat{\mathsf{G}}_\reg$};

  \draw[->] (Vinf) -- node[midway, above=1pt, align=center]
                       {\shortstack{learning term\\Eq.~\eqref{eq:main:e2e_decomposition}}} (EVinf);
  \draw[->, thick] (Vinf) -- node[midway, text width=2.2cm, align=left]  {Theorem~\ref{thm:end-to-end-error}}(EVdtT);
  \draw[->, dashed] (G)  -- (Vinf);
  \draw[->, dashed] (GN) -- (EVdtT);
  \draw[->] (G) -- node[midway, text width=2.2cm, align=left] {Corollary~\ref{cor:RateEst}}(GN);
  \draw[->] (EVT) -- node[midway, below=1pt, align=center]
                    {\shortstack{discretization\\Eq.~\eqref{eq:main:e2e_decomposition}}} (EVdtT);
  \draw[->] (EVinf) -- node[midway, left=1pt, align=right]
                       {\shortstack{convergence\\Eq.~\eqref{eq:main:e2e_decomposition}}} (EVT);
\end{tikzpicture}}
    \caption{End-to-end error decomposition in \eqref{eq:main:e2e_decomposition} underlying \Cref{thm:end-to-end-error} for bounding the term $\|V^\star-\widehat V_{k}\|_{\LIN}$. }
    \label{fig:main:errorDecomposition}
\end{figure}

The total error can be bounded by the triangle inequality, yielding the following three-term decomposition, with all norms interpreted in $\LIN$:
\begin{flalign}
&\|V^\star\!-\!\widehat V_{k}\|
\le \|\Vstar\!-\!\EV^\star\|+\|\EV^\star\!-\!\EV_T\|+\|\EV_T\!-\!\widehat V_{k}\|.
&&\label{eq:main:e2e_decomposition}
\end{flalign}

This decomposition is  visualized in \Cref{fig:main:errorDecomposition}. The learning term compares the true optimal value function $V^\star$ with the steady state of the HJB induced by the learned generator $\EV^\star$. The convergence term measures how close the approximate continuous-time HJB flow is to its own steady state after horizon $T=k\Delta t$. The discretization term measures the error incurred by replacing the approximate HJB flow with the finite IMEX scan in \Cref{alg:O2RL}. Only the first term depends directly on statistical estimation, as the other two are controlled by stability and numerical integration.
Despite our nonlinear HJB setting, convergence of the full HJB can still be analyzed using linear operator convergence analysis
\cite{li2022optimal,Kostic2023KoopmanEigenvalues,kostic2024consistent}. We use operator-norm error analysis \cite{Kostic2023KoopmanEigenvalues}, which captures worst-case behavior essential for reliable policies, unlike the average-case nature of Hilbert-Schmidt bounds. With the inclusion $\injectOUT:\spOUT\hookrightarrow\LOUT$ made explicit, the aggregate generator error is
\begin{align}
\error(\widehat{G}_\reg)
\defeq
\norm{\target-\injectOUT\widehat{G}_\reg}_{\spIN\to\LOUT}.
\label{eq:main:generator_error}
\end{align}
For every approximation tolerance $\varepsilon>0$, there exists a finite-rank ${\mathsf{G}}_\reg\in\HS{\spIN,\spOUT}$ such that $\error({\mathsf{G}}_\reg)<{\mathsfit B}(\spOUT)+\varepsilon$. If $\spIN$ is chosen as a $C_0$-universal RKHS~\cite[Definition~1]{sriperumbudur2010relation} (a property satisfied by many common kernels), then, consistent with the representation/projected-estimation split in \eqref{eq:main:risk_decomp}, arbitrarily accurate finite-rank approximations of the infinitesimal generator are available and the representation bias vanishes, ${\mathsfit B}(\spOUT)=0$, thereby removing the output-space projection error \cite{mollenhauer2020nonparametric,Kostic2023KoopmanEigenvalues,bevanda2025nonparametric}, which we therefore assume below.
\begin{assumption}
\label{asm:ForThm}
 We assume that $\spIN$ is a $C_0$-universal RKHS, $\mathcal{D}_\bu(\cdot,\bx)$ is continuously differentiable with locally Lipschitz gradient and is globally Lipschitz in its first argument, uniformly in $\bx$ (see \Cref{prop:main:LipschitzFenchel}), and the state reward satisfies $r_{\bx}\in\spIN$.  We further assume that the exact discounted HJB \eqref{eq:discountHJB1}, with $\rho > 0$, has a well-defined solution $\Vstar \in H_\mu^m(\spX)$ that coincides with the optimal expected discounted return, as discussed in Section~\ref{sec:background-and-problem}.
\end{assumption}
\begin{remark}[Satisfying the assumptions]\label{rem:main:rkhsAugmentation}
For example, if $\mu$ has a bounded density, a $C_0$-universal Mat\'ern-$\nu$ RKHS with $\nu>2$ satisfies Assumption~V.1 with $m=2$~\cite{sriperumbudur2010relation}, \cite[Examples~2.6 and~2.8]{kanagawa2018gaussian}. Under standard regularity and integrability conditions for uniformly elliptic discounted HJB equations, one also obtains $\Vstar\in H_\mu^2$ as required by Assumption~VII.1~\cite[Chapter~IV, Section~5]{fleming2006controlled}. The remaining bounded-action requirement is handled by the smooth approximation of Remark~IV.3, and the reward condition $r_{\bx}\in\spIN$ by the following augmentation: if $r_{\bx}\notin\spIN$, we can augment the RKHS with the kernel $k^r(\bx,\bx')\coloneqq k(\bx,\bx')+r_{\bx}(\bx)r_{\bx}(\bx')$, ensuring that $r_{\bx}$ lies in the resulting space, which we still denote by $\spIN$.
\end{remark}
For the HJB perturbation, it is clearer to work with the residual operators themselves. Recall that $\injectIN:\spIN\hookrightarrow\LIN$ is the inclusion of state-only RKHS functions into the state-marginal space. We identify $(\spIN)^{n_a}$ with its componentwise inclusion in $(\LIN)^{n_a}$, equipped with the product norm. Define
\begin{align}
\Delta_A
&\defeq
\AbFPK_{\mid\spIN}-\injectIN\EAEstim_\reg,
\notag \quad
\Delta_B
\defeq
\BbFPK_{\mid\spIN}-\EBEstim_\reg,
\notag\\
\delta
&\defeq
\max\left\{
\norm{\Delta_A}_{\spIN\to\LIN},
\norm{\Delta_B}_{\spIN\to(\LIN)^{n_a}}
\right\}.
\label{eq:main:block_errors}
\end{align}
For $h\in\spIN$, the joint residual decomposes as
\begin{align}
[(\target\!-\!\injectOUT\widehat{\mathsf{G}}_\reg)h](\bx,\bu)
\!=\!
(\Delta_A h)(\bx)\!+\!\langle\bu,(\Delta_B h)(\bx)\rangle.
\label{eq:main:joint_block_residual}
\end{align}
Since $\mu$ is the state marginal of $\eta$, viewing the state-only term $\Delta_A h$ as a function on $\spZ$ that is constant in $\bu$ preserves its norm, namely $\tinynorm{\Delta_A h}_{\LOUT}=\tinynorm{\Delta_A h}_{\LIN}$. The following condition ensures that errors in the autonomous and controlled coefficients cannot cancel or remain unobserved in the joint residual.
\begin{assumption}[Block separation]\label{asm:block-separation}
There exists a constant $c_\eta<\infty$, independent of $N$, such that, for every $h\in\spIN$,
\begin{align}
\|\Delta_A h\|_{\LIN}^2\!+\!\|\Delta_B h\|_{(\LIN)^{n_a}}^2
\!\le\!c_\eta^2\|(\target\!-\injectOUT\widehat{\mathsf{G}}_\reg)h\|_{\LOUT}^2 .
\label{eq:main:block_separation}
\end{align}
\end{assumption}
Under this assumption, taking the supremum over $\tinynorm{h}_{\spIN}\le1$ gives the direct relation
\begin{align}
\delta
\le c_\eta\,\error(\widehat{\mathsf{G}}_\reg).
\label{eq:main:block_to_joint_error}
\end{align}
Here $c_\eta$ quantifies the coverage of the controlled directions under $\eta$ and their separation from the autonomous term, which is a condition on the variety of the recorded actions rather than on the quality of the behavior policy, so that $c_\eta$ remains finite as long as the logged actions are not concentrated on a fixed function of the state. It enters only the multiplicative constant in the rate.
Furthermore, the requirements on $\mathcal{D}_\bu$ and $r_{\bx}$ ensure that the iterates of \Cref{alg:O2RL} satisfy $\EV_k\in\spIN$, allowing us to bound $\|\EV_k-\Vstar\|_{\LIN}$.

\begin{lemma}[Exponential convergence of the HJB flow]
\label{lem:main:expConvHJB}
Suppose the discounted HJB flow generated by $\mathcal T$ admits a stationary solution $\Vstar$ and that the optimally controlled Fokker-Planck-Kolmogorov generator $\mathcal A^*+\mathcal B^*\bpi^\star$ has leading eigenvalue $0$ and spectral gap $\lambda_{\mathrm{gap}}\defeq-\operatorname{Re}(\lambda_2(\mathcal A^*+\mathcal B^*\bpi^\star))\ge0$, where the eigenvalues are ordered by decreasing real part. Then, for some $\bar C<\infty$, any solution $V(t)$ with $V(0)=V_0$ satisfies, modulo additive constants,
\[
\|V(t)-\Vstar\|_{\LIN}\le \bar C e^{-(\lambda_{\mathrm{gap}}+\rho)t}\|V_0-\Vstar\|_{\LIN}.
\]
\end{lemma}
The proof is given in the appendix. Since additive constants do not affect the policy, convergence hereafter is modulo constants. The spectral gap $\lambda_{\mathrm{gap}}$ governs the nonconstant components of the value function error, and discount increases their decay rate to $\lambda_{\mathrm{gap}}+\rho$.
	 
\begin{restatable}{theorem}{EtoEError}
\label{thm:end-to-end-error} 
Suppose that \Crefrange{asm:dynamics}{asm:ForThm}, \Cref{asm:block-separation}, and the conditions of
\Cref{prop:approxHJBflow}, \Cref{cor:IMEX}, and \Cref{lem:main:expConvHJB} hold.
Then, under the zero-initial condition $\EV_0 = 0$, provided that $\delta$ in \eqref{eq:main:block_errors} and $\Delta t$ are sufficiently small, and that $\sup_N\sup_{t\ge0}\tinynorm{\EV(t)}_{\spIN}<\infty$ (bounded RKHS norm of the learned value iterates), the learned HJB flow \eqref{eq:finite_hjb_evol} admits a stationary point $\EVstar$ to which it converges at the rate $\widehat{\lambda}_{\mathrm{gap}}+\rho>0$, and
\begin{align}
\!\!\!\!\norm{\Vstar - \EV_k}_\LIN
& \leq
C
\underbrace{(\lambda_{\mathrm{gap}}{+}\rho)^{-1} \delta}_{ \mathrm{learning}} + \underbrace{\mathcal{O}(\Delta t)}_{\mathrm{discretization}}
\notag\\
&\quad
+ C\underbrace{e^{-(\widehat{\lambda}_{\mathrm{gap}}{+}\rho) k \Delta t } \|\EV^\star\|_\LIN}_{\mathrm{convergence}} 
\label{eq:end2enderror_rate}
\end{align}
\label{eq:end2enderror}
where $C<\infty$ is independent of $N$, $\delta$, $k$, and $\Delta t$. Here $\widehat\lambda_{\mathrm{gap}}+\rho\coloneqq-\max\{\operatorname{Re}z:z\in\sigma(D\bm T(\bv^\star))\}>0$, where $D\bm T(\bv^\star)$ is the Jacobian of the learned HJB flow at its stationary point.
\end{restatable}
\begin{IEEEproof}[Proof sketch]
Using the decomposition in \eqref{eq:main:e2e_decomposition}, the first-order IMEX scheme gives the discretization term $\mathcal O(\Delta t)$. The stationary learning error follows by evaluating the exact HJB residual at $\EVstar$, expanding around $\Vstar$, and applying the exact resolvent bound with rate $\lambda_{\mathrm{gap}}+\rho$. On the finite-dimensional learned span, the same perturbation bound applies to the first derivative. Robustness of the exact stable flow therefore gives the decay rate $\widehat{\lambda}_{\mathrm{gap}}+\rho$ of the learned HJB flow, while Gronwall's lemma controls the finite-time transient into its local stability neighborhood. A detailed proof is provided in Appendix~\ref{app:sec:proofLearning}.
\end{IEEEproof}

To turn \eqref{eq:main:block_to_joint_error} into a sample-complexity statement, we invoke the operator-norm KRR analysis of~\cite{Kostic2023KoopmanEigenvalues}. In that notation, \textbf{(RC)} is a source/alignment condition saying that the target generator is sufficiently compatible with the covariance geometry of the RKHS, \textbf{(BK)} bounds the feature maps and controls estimator concentration, and \textbf{(SD)} imposes eigenvalue decay of $\mathsf{C}_{\mathsf{z}\mathsf{z}}$. The detailed assumptions are kept in the appendix, here they enter only through the resulting rate for $\error(\widehat{\mathsf{G}}_\reg)$.

\begin{restatable}{corollary}{Rate}
\label{cor:RateEst}
Under \Cref{asm:block-separation}, assume the conditions \textbf{(RC)}, \textbf{(BK)}, and \textbf{(SD)} of~\cite{Kostic2023KoopmanEigenvalues} hold for the KRR estimator, with source regularity $\rpar \in (1,2]$ and covariance spectral-decay exponent $\spar \in (0,1]$, and choose the regularization $\reg$ of order $N^{-\frac{1}{\rpar + \spar}}$. Then, for any $\xi \in (0,1)$ there exists $C>0$ such that, with probability at least $1-\xi$ over the i.i.d. samples, the learning error in \eqref{eq:end2enderror_rate} satisfies
    \begin{align}
       \norm{\Vstar - \EV^\star}_\LIN \leq  C \,(\lambda_{\mathrm{gap}}{+}\rho)^{-1}  \bigl(N^{-\frac{\rpar}{2(\rpar + \spar)}}\bigr)\,\ln \xi^{-1}.
    \end{align}
\end{restatable}

This follows by combining the KRR rate of \cite[Theorem~5.4]{kostic2024consistent} and \cite[Theorem~2]{Kostic2023KoopmanEigenvalues} for $\EEstim_\reg$ with the block relation \eqref{eq:main:block_to_joint_error} and then substituting the resulting $\delta$ into \Cref{thm:end-to-end-error}. Together with \Cref{cor:RateEst}, this highlights interpretable quantities that govern the difficulty of continuous-time offline RL: the amount of data, the RKHS alignment of the generator, the spectral decay of the state-action covariance, and the stability margin of the learned closed-loop dynamics. Informally, larger $\rpar$ means better source regularity/alignment, while smaller $\spar$ corresponds to faster eigenvalue decay. Both effects improve sample complexity, and in the favorable limit the statistical term approaches the $N^{-1/2}$ rate.

These guarantees are in-distribution, i.e., the sampling distribution $\eta$ determines both the state-action covariance geometry and the $L^2_\eta$ generator error.

\paragraph*{Approximating Infinitesimal Data}
When exact drift measurements \eqref{eq:dataset} are not directly available, fixed-step transition samples can be converted into proxy targets by the finite difference
\(
\bm d_{\Delta\tau}^{(i)}
\defeq \Delta\tau^{-1}(\bx_+^{(i)}-\bx^{(i)})
\). 
Here $\bx_+^{(i)}$ is the state observed one sampling interval $\Delta\tau$ after $\bx^{(i)}$, with the recorded action $\bu^{(i)}$ held fixed over that interval. Thus $\bm d_{\Delta\tau}^{(i)}$ is a finite-difference proxy for the drift measurement $\bm d^{(i)}$ in \eqref{eq:dataset}.
For the following explicit finite-difference result, we specialize to state-independent diffusion, i.e.,
\(\bm{\Sigma}(s)=\sqrt{2\epsilon}\,I\). Under this assumption, the
second-order generator term in \eqref{eq:autonmGen} reduces to \(\epsilon\Delta_s\).
The following corollary records the additional target error in the same i.i.d.\ setting as \Cref{cor:RateEst}.

\begin{restatable}{corollary}{FiniteDifferenceRate}
\label{cor:finite_difference_rate}
Under the conditions of \Cref{cor:RateEst}, assume additionally that the samples are i.i.d.\ one-step transitions, $\eta$ has compact
support and diffusion \(\bm{\Sigma}(s)=\sqrt{2\epsilon}\,I\). Moreover, for each fixed action \(\bu\), assume uniformly bounded conditional first-moment generators for all components of \(\bm d\) over one-step intervals. If $\reg$ is chosen of order
$N^{-\frac{1}{\rpar+\spar}}$, then the learning error in \eqref{eq:end2enderror_rate} satisfies
\begin{equation}
\begin{aligned}
\norm{\Vstar-\EV^\star}_{\LIN}
&=(\lambda_{\mathrm{gap}}{+}\rho)^{-1}
O_{\mathbb P}\!\Biggl(
N^{-\frac{\rpar}{2(\rpar+\spar)}}
\\[-0.2em]
&\qquad
+ N^{\frac{1}{\rpar+\spar}}
\biggl[
\Delta\tau+\sqrt{\frac{\epsilon}{N\Delta\tau}}
\biggr]\Biggr).
\end{aligned}
\label{eq:main:finite_difference_rate}
\end{equation}
\end{restatable}
The proof is given in Appendix~\ref{app:sec:proofLearning}.

\begin{remark}[Finite-difference stepsize]
\label{rem:finite_difference_stepsize}
In most offline RL datasets, $\Delta\tau$ is fixed by the logger or simulator, so if the sampling interval can be chosen, the finite-difference term simply balances bias against stochastic noise in the second term of \eqref{eq:main:finite_difference_rate}. Optimizing over $\Delta\tau$ gives $\Delta\tau^\star$ of order $N^{-1/3}$, and relative to the exact-target rate in \Cref{cor:RateEst}, this gives an additional finite-difference contribution
\[
O_{\mathbb P}\!\left(
N^{\frac{1}{\rpar+\spar}-\frac{1}{3}}
\right).
\]
Since \Cref{cor:RateEst} assumes $\rpar\in(1,2]$ and $\spar\in(0,1]$, we have $\rpar+\spar\le3$, hence the exponent $\frac{1}{\rpar+\spar}-\frac{1}{3}$ is nonnegative. Under the exact-target regularization schedule, the optimized finite-difference perturbation is therefore at best independent of $N$ at the endpoint $(\rpar,\spar)=(2,1)$, and otherwise grows with $N$. This should be read as a limitation of using transition-based finite-difference targets with the exact-target choice of $\reg$ as the raw finite-difference noise decreases as $N^{-1/3}$ after optimizing $\Delta\tau$, but the inverse-regularization factor $\reg^{-1}$ cancels that decrease.

\end{remark}

\section{NUMERICAL EXAMPLES}\label{sec:implementation-examples}
\subsection{COMPUTATIONAL COST} \label{subsec:compute}

The implementation of \Cref{alg:O2RL} is dominated by two matrix factorizations, one for the regularized Gram matrix in the world-model step and one for the IMEX propagator in the dynamic-programming step. This gives cubic scaling in the number of samples for the direct kernel implementation, while each HJB iteration requires a controlled-generator multiplication, a pointwise Fenchel-dual evaluation, and one application of the pre-factorized propagator. We summarize these overall computational complexities in \Cref{tab:complexities}. Unlike grid-based HJB solvers, whose state-space discretizations grow exponentially with the state dimension, the explicit dependence of the direct implementation on the state dimension is linear for fixed sample size and isotropic diffusion as used in our experiments. The dominant cost is instead dense kernel linear algebra in the number of samples. This cost is a limitation of the present realization rather than of the operator formulation itself. We return to scaling directions in \Cref{sec:discussion}.

\begin{table}[!ht]
\centering
\small
\setlength{\tabcolsep}{4pt}
\renewcommand{\arraystretch}{1.0}
\caption{Computational complexity of \Cref{alg:O2RL}. $C(n_a)$ denotes the complexity of evaluating one sampled Fenchel-conjugate term $\mathcal D_\bu(\bl(\bxi),\bxi)$, e.g.\ $C(n_a)=n_a$ for separable action penalties or $n_a^2$ for quadratic (coupled) action penalties.}
\label{tab:complexities}
\begin{tabular}{@{}lc@{}}
\toprule
\textbf{Task} & \textbf{Complexity} \\
\midrule
World Model Construction & $\mathcal{O}\!\big(N^3 + N^2(n_s + n_a)\big)$ \\
Dynamic Programming      & $\mathcal{O}\!\big(N^3 + k_{\max}(N^2 n_a + N\,C(n_a))\big)$ \\
\bottomrule
\end{tabular}
\end{table}

\begin{figure*}[!t]
    \centering
    \begin{minipage}[t]{0.47\textwidth}
        \centering
        \includegraphics[height=0.4\linewidth,trim={0 0 0 20pt},clip]{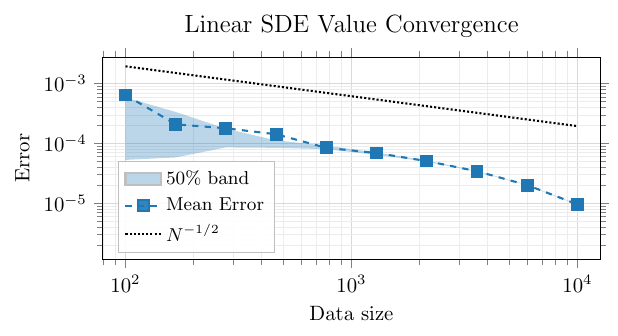}
        \par\smallskip\footnotesize (a) Linear SDE value convergence
    \end{minipage}\hspace{\fill}
    \begin{minipage}[t]{0.47\textwidth}
        \centering
        \includegraphics[height=0.4\linewidth,trim={0 0 0 20pt},clip]{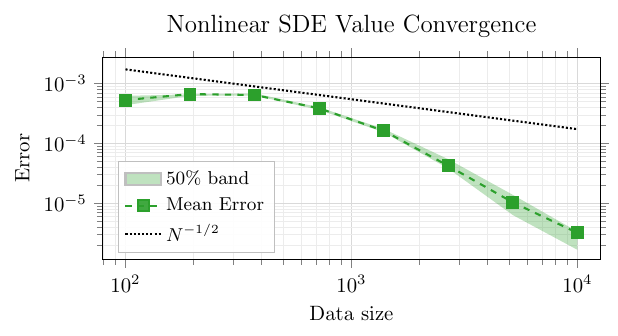}
        \par\smallskip\footnotesize (b) Nonlinear SDE value convergence
    \end{minipage}
    \caption{%
        Value function learning for one-dimensional benchmark systems. 
        The error convergence confirms our worst-case analysis.
    }
    \label{fig:rate}
\end{figure*}

\subsection{EXPERIMENTAL DESIGN}
All numerical examples use the same pipeline in \Cref{alg:O2RL}: construct infinitesimal samples, learn the generator blocks $(\SFA,\SFB)$, evolve the IMEX value iteration, and extract the policy from the final costate. When data are available as transitions with environment step $\Delta t_{\mathrm{env}}$, we use a finite-difference proxy for the infinitesimal samples. We use three groups of experiments with different goals. First, the one-dimensional controlled-diffusion examples are designed to illustrate the finite-sample behavior of the bounds under i.i.d.\ infinitesimal data, additive diffusion, strongly convex quadratic action penalties, and known optimal value functions. Second, the \texttt{Pendulum-v1} experiment evaluates \(\textsf{O-CTRL}\) in a more realistic offline RL setting using transition datasets, where the i.i.d.\ infinitesimal-sample assumption need not hold. This benchmark requires a genuinely nonlinear swing-up maneuver rather than only local stabilization near the upright equilibrium and allows comparison against established offline RL baselines implemented using the \texttt{d3rlpy} library with default settings. Third, the nonlinear four-dimensional AUV experiment reuses one reward-free generator model for depth regulation and vertical-speed glide without new data.

\subsection{FINITE-SAMPLE RATE EVALUATION}

We evaluate our learning error rates on linear and nonlinear process dynamics (Fig.~\ref{fig:rate}). While these are often studied using different policy classes, linear and nonlinear \cite{tu2019gap}, our convergence analysis covers both on equal footing -- without any parametric assumptions \cite{recht2019tour}. We run our proposed \Cref{alg:O2RL} over 8 seeds and gather i.i.d.\ data to form the quantiles and means in Fig.~\ref{fig:rate}. For these two examples, we use a Gaussian Radial Basis Function (RBF) kernel $k(\bm{x},\bm{y}) = \exp{(\textstyle -\nicefrac{\|\bm{x} -\bm{y}\|^ 2}{2\sigma^2})}$ while the data samples are drawn randomly from the set $[-3,3] \times [-3.5,3.5]$. Before comparison with the known undiscounted ergodic references, we correct the additive offset of each learned value function by imposing $V(0)=0$.

\noindent\textbf{Linear SDE with additive action.}
Linear-Quadratic (LQ) control problems play an important role in the control literature. They provide explicit solutions and, often, nonlinear ones can be approximated by LQ ones~\cite{zhao2023policy}.
To validate our findings, we study value-function convergence for the Ornstein--Uhlenbeck process
\begin{align*}
    \mathrm{d}S_t = (-S_t+a_t)\,\mathrm{d}t+\sqrt{2\epsilon}\,\mathrm{d}W_t,
\end{align*}
with reward \(r(s,a)=-3s^2-a^2\). We use a small discount \(\rho=10^{-3}\). In the vanishing-discount limit, the offset-corrected discounted value function approaches the ergodic reference
\(V_\infty^\star(s)=-s^2\), obtained by rewriting the cost formulation of~\cite{houska2025convex} as a reward-maximization problem. The limiting optimal policy is given by \(a^\star(s)=-s\). We set \(\epsilon=0.01\) and \(\sigma=10\) as the hyperparameter of the Gaussian RBF kernel.

\smallskip
\noindent\textbf{Nonlinear SDE with affine action.}
We transfer to a nonlinear setting, using an action-affine benchmark system
\begin{align*}
    \mathrm{d}S_t = (f(S_t)+g(S_t)a_t)\,\mathrm{d}t+ \sqrt{2\epsilon}\,\mathrm{d}W_t,
\end{align*}
where \(g(s)=\frac{1}{2} + \sin(s)\) and \(f(s)=-\frac{s}{2}(1-g(s)^2)\)~\cite{Doyle1996}. Here, the reward is \(r(s,a)=-s^2-a^2\). We use the same small discount \(\rho=10^{-3}\) and compare the offset-corrected learned value function with the vanishing-discount ergodic reference \(V_\infty^\star(s)=-s^2\). We set \(\sigma=1\) as the hyperparameter of the Gaussian RBF kernel and \(\epsilon=0.01\).
\Cref{fig:main:policy_value_linear} and \Cref{fig:main:policy_value_nonlinear} show the learned value functions and policies for both one-dimensional systems.

\subsection{OFFLINE RL BENCHMARK: PENDULUM}
We evaluate \Cref{alg:O2RL} on Gymnasium \texttt{Pendulum-v1}~\cite{towers2024gymnasium} as a practical offline RL benchmark using transition datasets rather than i.i.d.\ infinitesimal samples.
The action is a torque \(a\in[-2,2]\), and the observation is \(\bx=(\cos\theta,\sin\theta,\dot\theta)\) with \(\cos\theta,\sin\theta\in[-1,1]\) and \(\dot\theta\in[-8,8]\). Episodes truncate at 200 steps. Following the official reward, we split the state term and action cost as \(r_{\bx}(\theta,\dot\theta)=-(\theta^{2}+0.1\,\dot\theta^{2})\) and \(c_{\bu}(a)=0.001\,a^{2}\), so the maximum achievable reward is \(0\) (upright, zero velocity, zero torque). We run \Cref{alg:O2RL} with \(\rho=0.1\), \(\sigma=3\), \(\gamma=10^{-7}\), and \(k_{\max}=1000\). Since \texttt{Pendulum-v1} is deterministic, we add an artificial diffusion $\bm\Sigma=\sqrt{2\epsilon}\,I$ with $\epsilon=0.01$, contributing an $\epsilon\Delta$ smoothing term to the generator target. For training, we use two offline \texttt{d3rlpy} datasets (``Random'', ``Replay''), subsample \(8000\) \((\text{state},\text{action},\text{next state})\) tuples, and obtain infinitesimal samples via finite differences. The learned Pendulum value function for the replay dataset is already visualized in \Cref{fig:pendulumPlot}. \Cref{fig:main:pendulum_policy} shows the corresponding nonlinear policy and provides qualitative evidence that \(\textsf{O-CTRL}\) recovers nonlinear feedback from offline data. \Cref{tab:pendulum} reports mean \(\pm\) standard deviation of episode return over 50 i.i.d.\ rollouts of length 200. $\textsf{O-CTRL}$ was trained on the same \(8\mathrm{k}\) samples as the baselines per dataset.

\noindent \textbf{Discussion}
Offline RL faces distribution shift and extrapolation error, where policies trained on fixed data may query actions outside the data support, causing overestimation and instability. \emph{Conservative Q-Learning (CQL)} \cite{kumar2020conservative} uses a pessimistic objective suited to suboptimal data and, in our results, does better on Random, where the experience is collected with a random policy, than on Replay (\Cref{tab:pendulum}). \emph{Twin Delayed Deep Deterministic Policy Gradient with Behavior Cloning (TD3+BC)} \cite{fujimoto2021minimalist} anchors the policy to the behavior data and is sensitive to dataset quality, resulting in overall underperformance, as neither dataset is expert-level, although it improves on Replay. \emph{Implicit Q-Learning (IQL)} \cite{kostrikov2022offline} favors actions that look better than average within the dataset, consistent with stronger results on Replay. These methods provide a practical baseline suite for evaluating
\(\textsf{O-CTRL}\) under identical offline data and evaluation (Table~\ref{tab:pendulum}). As additional context, we also report a Soft Actor-Critic (SAC) expert trained online and evaluated with the same protocol. \(\textsf{O-CTRL}\) is competitive on Random and achieves a
higher mean return than the reported SAC expert on Replay. Note that, unlike value-based offline RL methods tied to the behavior distribution, \(\textsf{O-CTRL}\) first learns a reward-free world model from fixed data and then optimizes the task objective in that model. This decoupling makes performance less sensitive to whether the dataset is expert, medium, or random, and allows reusing the same learned model for different rewards without additional data collection. These properties make \(\textsf{O-CTRL}\) a promising direction for offline RL.

\begin{figure}[t!]
    \centering
    \includegraphics[width=\linewidth]{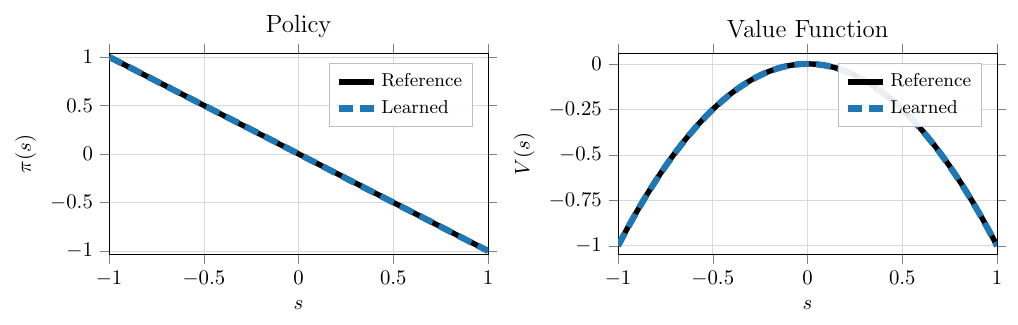}\par
    \caption{Learned value function and policy for the linear SDE with additive action using $200$ training samples.}
    \label{fig:main:policy_value_linear}
\end{figure}

\begin{figure}[t!]
    \centering
    \includegraphics[width=\linewidth]{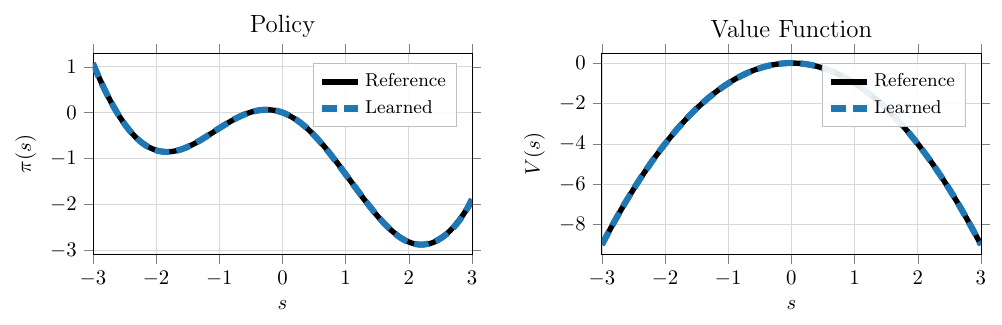}\par
    \caption{Learned value function and policy for the nonlinear SDE with affine action using $200$ training samples.}
    \label{fig:main:policy_value_nonlinear}
\end{figure}

\begin{figure}[t!]
    \centering
    \includegraphics[width=0.5\linewidth]{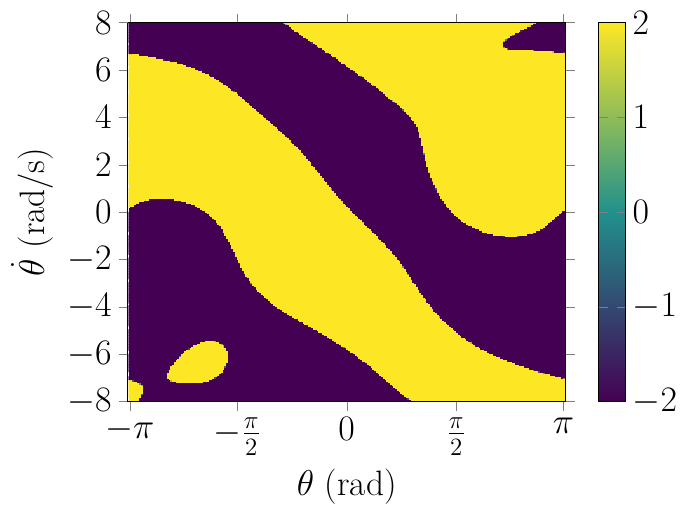}
    \caption{Nonlinear policy obtained by $\textsf{O-CTRL}$ on the \texttt{Pendulum-v1} replay dataset using $8000$ offline transition samples converted to infinitesimal tuples \eqref{eq:dataset}.}
    \label{fig:main:pendulum_policy}
\end{figure}

\subsection{AUTONOMOUS UNDERWATER VEHICLE CONTROL}
We evaluate \(\textsf{O-CTRL}\) on the nonlinear four-dimensional AUV model in~\cite{naik2007state}. The state is \(\bx=[w,q,z,\theta]^\top\), where \(w\) is heave velocity, \(q\) is pitch rate, \(z\) is depth, and \(\theta\) is pitch angle. The input is the stern-plane angle \(\delta_s\), and the forward speed is fixed at \(u_0=2\,\mathrm{m/s}\). For each of five seeds, we train a model on \(n=8000\) i.i.d.\ samples from \(w\in[-0.5,0.5]\,\mathrm{m/s}\), \(q\in[-\pi/6,\pi/6]\,\mathrm{rad/s}\), \(z\in[0,4]\,\mathrm{m}\), \(\theta\in[-\pi/6,\pi/6]\,\mathrm{rad}\), and \(\delta_s\in[-\pi/6,\pi/6]\,\mathrm{rad}\). We use state-independent diffusion with \(\epsilon=10^{-3}\), regularization \(\gamma=10^{-8}\), standardized state coordinates, and an RBF length scale selected by the median heuristic~\cite{garreau2017large}.

We use each learned model for two tasks. For depth regulation at \(z=2\,\mathrm{m}\), we split the reward into \(r_{\bx}^{\mathrm{depth}}(\bx)=-(20w^2+100q^2+300(z-2)^2+50\theta^2)\) and \(c_{\bu}^{\mathrm{depth}}(\delta_s)=125\delta_s^2\). For glide at the constant vertical speed \(\dot z=0.1\,\mathrm{m/s}\), where \(\dot z=w\cos\theta-u_0\sin\theta\), we use \(r_{\bx}^{\mathrm{glide}}(\bx)=-(10q^2+120(\dot z-0.1)^2)\) and \(c_{\bu}^{\mathrm{glide}}(\delta_s)=50\delta_s^2\). Note that the glide objective is nonquadratic in the state through \(\sin\theta\) and \(\cos\theta\). For both tasks, the discounted HJB recursion uses \(\rho=1\), \(\Delta t=0.005\), and \(k_{\max}=2000\).

 We evaluate each controller on the true nonlinear AUV with an Euler--Maruyama step of \(0.01\,\mathrm{s}\) and \(\epsilon=10^{-3}\). The policy is evaluated at the current state before each step. For each learned model, we run 50 depth trials of \(30\,\mathrm{s}\) from \([0,0,0,0]^\top\) and 50 glide trials of \(10\,\mathrm{s}\) from \([0,0,1,0]^\top\). \Cref{fig:main:auv_tasks} shows the pooled mean and population standard deviation over all 250 trials for each task. Finally, for each learned model, we compute the root-mean-square error (RMSE) over all 50 trajectories and sampled time points. After the transient, the depth RMSE is evaluated relative to \(z=2\,\mathrm{m}\) over \([10,30]\,\mathrm{s}\), and the vertical-speed RMSE is evaluated relative to \(\dot z=0.1\,\mathrm{m/s}\) over \([4,10]\,\mathrm{s}\). Across the five model seeds, the mean \(\pm\) sample standard deviation is \(0.0950\pm0.0086\,\mathrm{m}\) for depth and \(0.0861\pm0.0016\,\mathrm{m/s}\) for glide. The two tasks reuse the same learned infinitesimal world model without new data.

\begin{figure*}[t!]
    \centering
    \begin{minipage}[t]{0.249\textwidth}
        \centering
        \includegraphics[width=\linewidth]{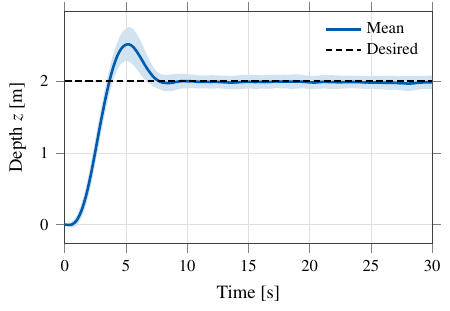}
        \par\smallskip\footnotesize (a) Depth regulation
    \end{minipage}\hspace{0.001\textwidth}%
    \begin{minipage}[t]{0.249\textwidth}
        \centering
        \includegraphics[width=\linewidth]{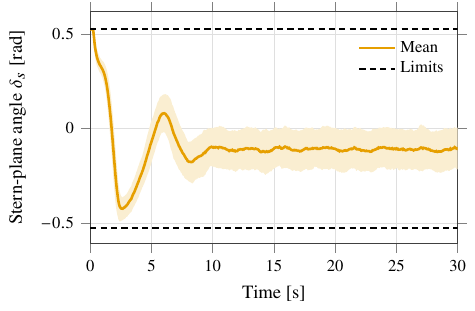}
        \par\smallskip\footnotesize (b) Depth control input
    \end{minipage}\hspace{0.001\textwidth}%
    \begin{minipage}[t]{0.249\textwidth}
        \centering
        \includegraphics[width=\linewidth]{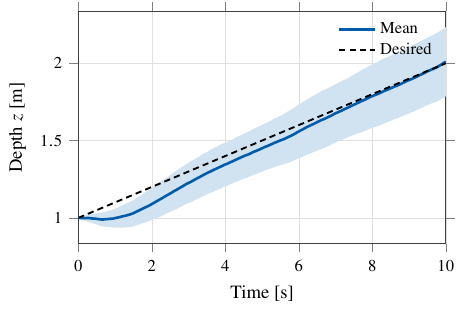}
        \par\smallskip\footnotesize (c) Glide depth
    \end{minipage}\hspace{0.001\textwidth}%
    \begin{minipage}[t]{0.249\textwidth}
        \centering
        \includegraphics[width=\linewidth]{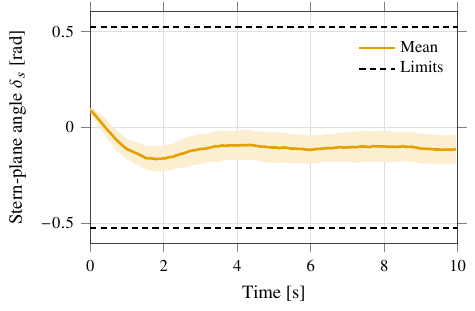}
        \par\smallskip\footnotesize (d) Glide control input
    \end{minipage}
    \caption{AUV depth and control trajectories for the two objectives. Solid curves show the pooled mean over five learned models and 50 trials per model. Shaded regions show one population standard deviation. Dashed lines show the target or input limits.}
    \label{fig:main:auv_tasks}
\end{figure*}

\begin{table}[t]
\centering
\small
\setlength{\tabcolsep}{6pt}
\caption{Pendulum results on \texttt{d3rlpy} offline datasets (Random, Replay). 
Higher (less negative) is better. The expert policy is a Stable-Baselines3 Soft Actor-Critic (SAC) 
agent (Pendulum-v1) trained online (dataset not applicable).\protect\footnotemark}
\begin{tabular}{lcc}
\toprule
& \multicolumn{2}{c}{Mean Episode Reward ($\pm$ std)} \\
\cmidrule(lr){2-3}
Algorithm & Random & Replay \\
\midrule
\textbf{Ours} & $\mathbf{-141.71 \pm 84.23}$ & $\mathbf{-104.36 \pm 83.00}$ \\
CQL           & $-293.09 \pm 206.68$         & $-451.00 \pm 530.91$ \\
IQL           & $-313.00 \pm 222.91$         & $-284.89 \pm 196.59$ \\
TD3+BC        & $-1181.08 \pm 347.40$        & $-748.11 \pm 420.18$ \\
\midrule
Expert        & \multicolumn{2}{c}{$-130.98 \pm 79.40$} \\
\bottomrule
\end{tabular}
\label{tab:pendulum}
\end{table}
\footnotetext{\url{https://huggingface.co/sb3/sac-Pendulum-v1}}

\section{CONCLUSION}\label{sec:discussion}
We studied offline reinforcement learning with continuous-time dynamics and continuous state-action spaces. We introduced a model-based algorithm with a reward-free generator-learning stage that splits value function learning into two steps. The first step is learning the generator of the controlled diffusion process from data through operator regression in RKHS. The second step is to solve the Hamilton–Jacobi–Bellman Partial Differential Equation (PDE) in the RKHS. We analyze the method by propagating operator learning errors through the empirical HJB and prove that the resulting time-stepping scheme is consistent and convergent. Operator-norm error analysis yields explicit learning rates that relate data size and kernel smoothness to value function suboptimality. The bounds reveal how the generator's spectral gap, discount exponent, and regularity influence the offline RL problem. Examples confirm the predicted rates.

\smallskip
\textbf{Limitations and Open Questions}
Although the proposed framework provides insight into continuous-time offline RL, several limitations and open questions remain.

\textit{Broader framing~} Control-affine dynamics are a natural starting point since they give a tractable HJB update and a convex pointwise control problem. More general nonlinear action dependence could use richer tensor-product action features, as in recent nonparametric operator methods~\cite{bevanda2025nonparametric}.

\textit{Data~} The rate in \Cref{cor:RateEst} uses i.i.d.\ sampling only to turn the abstract world-model error $\delta$ in \eqref{eq:main:block_errors} into the explicit KRR rate of~\cite{Kostic2023KoopmanEigenvalues}. The HJB propagation result in \Cref{thm:end-to-end-error} accepts any admissible bound on $\delta$, separating value-error analysis from the sampling model. Natural next steps are rates for dependent trajectory data~\cite{mirzaei2025empirical} and extensions to partial observation. Choosing $\eta$ optimally is a separate experimental-design problem that becomes central when extending the present $L^2$ analysis to uniform $L^\infty$ or off-distribution guarantees.

\textit{Analysis~} The current guarantees are worst-case upper bounds. Lower bounds could separate proof artifacts from intrinsic statistical or dynamical barriers to deliver \emph{minimax optimal} rates for offline reinforcement learning via matching upper and lower bounds. Data-dependent guarantees could also assess the concrete suboptimality of actions and guide episodic data collection.
    
\textit{Representations~} 
The universal RKHS estimator gives a general nonparametric analysis but dense implementations scale cubically. Empirical scalability will require approximate kernels such as Nystr\"om and sketching~\cite{meanti2023estimating}, random features~\cite{rahimi_random_2007}, or learned representations~\cite{kostic2024learning}. These tools are key for moving beyond small and medium-scale experiments.

\textit{Algorithm~} Infinite-horizon value convergence requires the learned HJB flow to remain stable. Since this margin and the associated norm distortion are usually unknown before learning, practical diagnostics, regularization, and adaptive schemes for maintaining them remain open.

\section*{APPENDIX}

\subsection{FENCHEL AND FINITE-DIMENSIONAL DERIVATIONS}\label{app:sec:octrl-derivations}

\begin{IEEEproof}[Proof of \Cref{prop:approxHJBflow}]
Let the empirical value be represented as
\[
\EV(t,\cdot)=\widehat{\mathsf{E}}_{\mathsf s}^{*}\bv(t),
\qquad
\EV(t,\bx)=\big\langle \bv(t),\bm k_{\mathsf s}(\bx)\big\rangle .
\]
For any coefficient vector $\bm c\in\mathbb R^N$, the adjoint sampling operator satisfies
\[
(\widehat{\mathsf{E}}_{\mathsf s}^{*}\bm c)(\bx)
=\big\langle \bm c,\bm k_{\mathsf s}(\bx)\big\rangle,
\qquad
\bm k_{\mathsf s}(\bx)=\widehat{\mathsf{E}}_{\mathsf s}\featx(\bx).
\]
Substituting the finite-dimensional representation of $\EV$ into the empirical HJB vector field
\[
\widehat{\mathcal T}(\EV)=-(\rho \mathsf{Id}-\EAEstim_\reg)\EV
\,+\,\widehat r_\bx
\,+\,\widehat{\mathcal D}_{\bu}(\EBEstim_\reg\EV,\cdot)
\]
and using the definitions
\[
\widehat r_\bx=\widehat{\mathsf{E}}_{\mathsf s}^{*}\bm r,\qquad
\widehat{\mathcal D}_{\bu}(\bl,\cdot)=\widehat{\mathsf{E}}_{\mathsf s}^{*}\bm D_{\bu}(\bl),
\]
gives
\begin{align}
\widehat{\mathcal T}(\EV)
&=-(\rho \mathsf{Id}-\EAEstim_\reg)\widehat{\mathsf{E}}_{\mathsf s}^{*}\bv
  +\widehat{\mathsf{E}}_{\mathsf s}^{*}\bm r
  +\widehat{\mathsf{E}}_{\mathsf s}^{*}\bm D_{\bu}(\EBEstim_\reg\widehat{\mathsf{E}}_{\mathsf s}^{*}\bv).
\label{eq:app:substitute_empirical_hjb}
\end{align}
Identifying $\EAEstim_{\reg}\widehat{\mathsf{E}}_{\mathsf s}^{*}\bv
=\widehat{\mathsf{E}}_{\mathsf s}^{*}\SFA\bv$ and using the KRR estimator formula for the autonomous block gives
\[
\EAEstim_{\reg}=\widehat{\mathsf{E}}_{\mathsf s}^{*}\bm K_\gamma^{-1}\ES_{\mathrm d},
\qquad
\SFA=\bm K_\gamma^{-1}\ES_{\mathrm d}\widehat{\mathsf{E}}_{\mathsf s}^{*},
\]
The controlled block has the same coefficient structure, but multiplied by the sampled actions. Action-affinity gives
\[
\EBEstim_{\reg}\widehat{\mathsf{E}}_{\mathsf s}^{*}\bv
=\big(\bm I_{n_a}\otimes \widehat{\mathsf{E}}_{\mathsf s}^{*}\big)\SFB\bv,
\]
where
\[
\SFB=\begin{bmatrix}
\mathrm{diag}(\widehat{\mathsf{U}}\bm{e}_1)\SFA\\
\vdots\\
\mathrm{diag}(\widehat{\mathsf{U}}\bm{e}_{n_a})\SFA
\end{bmatrix}\in\mathbb R^{Nn_a\times N},
\]
with $\widehat{\mathsf{U}}\bm e_k$ equivalently denoting the $k$-th action column in the sample matrix $\bm U$. Thus, for each sample $\bxi$, the costate argument of the Fenchel dual is
\[
[(\EBEstim_{\reg}\EV)(\bxi)]
=\big\langle \SFB\bv,\bm I_{n_a}\otimes\bm k_{\mathsf s}(\bxi)\big\rangle .
\]
Consequently, $\bm D_{\bu}(\SFB\bv)$ abbreviates
\[
\bm D_{\bu}(\SFB\bv)
=\bm K_\gamma^{-1}
\left[
\mathcal D_{\bu}\!\left(
\big\langle \SFB\bv,\bm I_{n_a}\otimes\bm k_{\mathsf s}(\bxi)\big\rangle,\bxi
\right)
\right]_{i\in[N]}.
\]
Inserting these identities into \eqref{eq:app:substitute_empirical_hjb} yields
\[
\widehat{\mathcal T}(\EV)
=\widehat{\mathsf{E}}_{\mathsf s}^{*}
\left[-(\rho\bm I-\SFA)\bv+\bm r+\bm D_{\bu}(\SFB\bv)\right].
\]
Applying the reproducing property once more gives, for every test point $\bx$,
\[
\widehat{\mathcal T}(\EV)(\bx)
=\big\langle
-(\rho\bm I-\SFA)\bv+\bm r+\bm D_{\bu}(\SFB\bv),
\bm k_{\mathsf s}(\bx)
\big\rangle .
\]
Finally, the autonomous matrix can be written in terms of the target kernel matrix as
\[
\SFA=\bm K_\gamma^{-1}\ES_{\mathrm d}\widehat{\mathsf{E}}_{\mathsf s}^{*}
=\bm K_\gamma^{-1}\Kyx,
\qquad
\Kyx\defeq\ES_{\mathrm d}\widehat{\mathsf{E}}_{\mathsf s}^{*}.
\]
Using the definition of $\ES_{\mathrm d}$, the It\^o formula, and the derivative reproducing property~\cite{arnold1974stochastic,kostic2024learningGenerator,zhang2023continuous}, the entries of this matrix are
\begin{align*}
(\Kyx)_{ij}
&=
\innerprod{\bm d^{(i)},\nabla_{\bxi}k(\bxi,\bxj)}
\notag\\
&\quad+
\frac{1}{2}
\operatorname{Tr}\!\left(
\bm{\Sigma}(\bxi)\bm{\Sigma}(\bxi)^\top
\nabla_{\bxi}^{2}k(\bxi,\bxj)
\right).
\end{align*}

Therefore the coefficient vector satisfies
\[
\dot\bv(t)=-(\rho\bm I-\SFA)\bv(t)+\bm r+\bm D_{\bu}(\SFB\bv(t))
\defeq \bm T(\bv(t)),
\]
and evaluating both sides against $\bm k_{\mathsf s}(\bx)$ yields \eqref{eq:finite_hjb_evol}.
\end{IEEEproof}

\begin{IEEEproof}[Proof of \Cref{cor:IMEX}]
Starting from the coefficient Ordinary Differential Equation (ODE) in \Cref{prop:approxHJBflow},
\[
\dot\bv(t)=-(\rho\bm I-\SFA)\bv(t)+\bm r+\bm D_{\bu}(\SFB\bv(t)),
\]
apply backward Euler to the stiff linear part and forward Euler to the nonlinear Fenchel term. With $\bv^{(k)}\approx\bv(k\Delta t)$ this gives
\[
\frac{\bv^{(k+1)}-\bv^{(k)}}{\Delta t}
=-(\rho\bm I-\SFA)\bv^{(k+1)}
+\bm r+\bm D_{\bu}(\SFB\bv^{(k)}).
\]
Collecting the implicit terms on the left yields
\[
\big(\bm I+\Delta t(\rho\bm I-\SFA)\big)\bv^{(k+1)}
=\bv^{(k)}+\Delta t\big(\bm r+\bm D_{\bu}(\SFB\bv^{(k)})\big).
\]
Multiplication by
\[
\bm M=\big(\bm I+\Delta t(\rho\bm I-\SFA)\big)^{-1}
\]
proves the stated IMEX recursion.
\end{IEEEproof}

\subsection{END-TO-END FINITE-SAMPLE GUARANTEES}\label{app:sec:proofLearning}

Throughout, $L$ denotes a Lipschitz constant, $c$ an absorbed constant and $C$ one that appears in a displayed bound, all finite and positive, with a bar for exact and a hat for learned objects. Constants are defined locally within each proof. We use standard parabolic regularity for the exact HJB on the compact support of $\mu$. Since $\bm G$ is continuous, it is bounded there, so $\BbFPK$ is bounded from $H_\mu^1(\spX)$ to $(\LIN)^{n_a}$, with operator norm $\tinynorm{\BbFPK}$. Moreover, $\AbFPK:H_\mu^m(\spX)\to\LIN$ is continuous and $\tinynorm{\,|\BbFPK h|^2}_{\LIN}\le C_{\BbFPK}\tinynorm{h}_{H_\mu^m}^2$, where $|\cdot|$ denotes the pointwise Euclidean norm on $\R^{n_a}$. Finally, under the standard uniformly elliptic realization associated with \Cref{asm:dynamics}, $\rho\mathcal I-\AbFPK$ is sectorial on $\LIN$, with $\operatorname{Re}z\ge\rho$ for every $z$ in its spectrum and the domain of $(\rho\mathcal I-\AbFPK)^{1/2}$ equal to $H_\mu^1(\spX)$.

\begin{IEEEproof}[Proof of \Cref{lem:main:expConvHJB}]
Let $\bpi^\star(\bx)=\bu^\star((\mathcal B\Vstar)(\bx),\bx)=\nabla_{\bl}\mathcal D_\bu((\mathcal B\Vstar)(\bx),\bx)$ be the stationary optimal policy. Hence, at the stationary solution,
\begin{align*}
\mathcal T (\Vstar)
&= -\big(\rho \mathcal{I}-\mathcal A\big)\Vstar + r_{\bx} + \mathcal D_{\bu}\big(\mathcal B \Vstar,\cdot\big) \\
&= (\mathcal A^*+\mathcal B^*\bpi^\star)^* \Vstar -\rho \Vstar + r_{\bx} - c_{\bu}(\cdot,\bpi^\star).
\end{align*}
Equivalently, the usual final-value HJB in physical time is written in the time-to-go variable $t$ as this ``descent'' evolution. It has the same steady state, preserves the correct parabolic sign, and enables stable forward integration in $t$.
Here, the operator $\mathcal A^*+\mathcal B^*\bpi^\star$ can be interpreted as
the Fokker-Planck-Kolmogorov (FPK) operator of the optimally controlled SDE;
see~\cite{houska2025convex}. Global convergence of the exact HJB flow follows from the FPK--HJB duality in~\cite{houska2025convex}, together with the exponential convergence of the optimally controlled FPK established in its Lemma~4. Since
$\bpi^\star$ is optimal and thus stationary, the Fréchet derivative of
$\mathcal T$ at $\Vstar$ satisfies
\[
\begin{aligned}
D_F^\star(\Vstar)
\defeq
\frac{\partial}{\partial V}\mathcal T(\Vstar)
&=
(\mathcal A^*+\mathcal B^*\bpi^\star)^*-\rho \mathsf{Id},\\
\sigma\!\left(D_F^\star(\Vstar)\right)
&=
\left\{-\rho+\overline{\lambda}:\lambda\in
\sigma(\mathcal A^*+\mathcal B^*\bpi^\star)\right\}.
\end{aligned}
\]
where $\sigma(\cdot)$ denotes the spectrum of a linear operator or matrix and $\bar{\lambda}$ denotes the complex conjugate of $\lambda$. The backward generator $(\mathcal A^*+\mathcal B^*\bpi^\star)^*$ annihilates constants, so $\mathbf 1$ is its zero eigenfunction. Since value functions are identified modulo additive constants, the leading relevant eigenvalue is therefore $\lambda_2$, whose decay is governed by $\lambda_{\mathrm{gap}}$ under~\cite[Lemma~4]{houska2025convex}. The discount shifts the spectrum by $-\rho$, increasing the decay rate of these modes to $\lambda_{\mathrm{gap}}+\rho$. Hence
\[
\|V(t)-\Vstar\|_{\LIN}
\le \bar C e^{-(\lambda_{\mathrm{gap}}+\rho)t}\|V_0-\Vstar\|_{\LIN},
\]
for some $\bar C<\infty$.
\end{IEEEproof}
\begin{remark}
The statement of \Cref{lem:main:expConvHJB} can be strengthened under additional regularity. On open bounded convex domains $\spX$ with Neumann boundary conditions
$\nabla V(\cdot,t)^\top {\bf \eta}\mid_{\partial \spX}=0$,
the generalized Bakry--Emery conditions of~\cite[Lemma~4]{houska2025convex} yield exponential ergodicity of the optimally controlled Fokker-Planck-Kolmogorov operator with a positive spectral gap $\lambda_\mathrm{gap}>0$. On unbounded convex domains, the same conclusion additionally requires the weak Has'minskii--Lyapunov condition in~\cite[Assumption~2]{houska2025convex}.
\end{remark}
\begin{IEEEproof}[Proof of \Cref{thm:end-to-end-error}]
With $T=k\Delta t$, we bound the three terms of \eqref{eq:main:e2e_decomposition} in turn, first the gap between the exact and the learned HJB operator, then the convergence of the learned flow, and finally the learning error $\tinynorm{\Vstar-\EV^\star}_\LIN$.

\smallskip
\noindent\textbf{HJB perturbation.}
Let $L_{\mathcal D}$ denote the uniform Lipschitz constant of $\mathcal D_\bu(\cdot,\bx)$ (see \Cref{prop:main:LipschitzFenchel} and \Cref{asm:ForThm}).
Fix a neighborhood of $\Vstar$ in $H_\mu^m(\spX)$ on which the costates $\BbFPK V$ stay in a bounded set. Subtracting the two HJB operators at any $V\in\spIN$ in this neighborhood, where the block errors of \eqref{eq:main:block_errors} are defined, leaves them together with the errors made in reconstructing the known reward and Fenchel-dual terms in the sample span,
\begin{align}
(\mathcal T-\widehat{\mathcal T})(V)
&=\Delta_A V
+\bigl[\mathcal D_{\bu}(\BbFPK V,\cdot)-\mathcal D_{\bu}(\EBEstim_\reg V,\cdot)\bigr]
\notag\\
&\quad+\bigl[r_{\bx}-\widehat r_{\bx}\bigr]
+\bigl[\mathcal D_{\bu}(\EBEstim_\reg V,\cdot)-\widehat{\mathcal D}_{\bu}(\EBEstim_\reg V,\cdot)\bigr].
\label{eq:app:operator_difference}
\end{align}
The last two terms are reconstructed using the same kernel, Gram matrix, and regularization as the generator estimator. Since $r_\bx\in\spIN$ and $\mathcal D_\bu$ and its first derivative in $V$ are smooth on the fixed compact neighborhood, their KRR approximation and regularization errors have the same order as the generator error and are absorbed into $\delta_{\mathcal T}$. Bounding the second term by $L_{\mathcal D}\|\Delta_B V\|_{(\LIN)^{n_a}}$ and using that $\delta$ is the maximum of the two block norms, the HJB-level perturbation is
\begin{align}
\delta_{\mathcal T}
&\defeq
\norm{\Delta_A}_{\spIN\to\LIN}
+L_{\mathcal D}\norm{\Delta_B}_{\spIN\to(\LIN)^{n_a}}
+\mathcal O(\delta)
\notag\\
&\le (1+L_{\mathcal D})\delta+\mathcal O(\delta)=\mathcal O(\delta).
\label{eq:app:hjb_operator_perturbation}
\end{align}

\smallskip
\noindent\textbf{Convergence of the approximate value function.}
Let $\widehat D_F(V)$ denote the Fr\'echet derivative of $\widehat{\mathcal T}$ at $V$. On the finite-dimensional range of $\widehat{\mathsf E}_{\mathsf s}^*$, we have that $\|h\|_{\spIN}\le c\|h\|_{\LIN}$, where $c$ also bounds $\tinynorm{W}_{\spIN}$ for $W$ in the neighborhood and absorbs its local constants. Differentiating \eqref{eq:app:operator_difference} in $V$, yields for $V$ and $h$ in that range
\begin{equation}
\|[D_F^\star(V)-\widehat D_F(V)]h\|_{\LIN}
\le c\,\delta_{\mathcal T}\|h\|_{\LIN}.
\label{eq:app:c1_hjb_perturbation}
\end{equation}
By \Cref{lem:main:expConvHJB} the exact flow converges globally, and its time-$\tau$ map is a contraction on a small ball around $\Vstar$ for some finite $\tau>0$. Fix $T'<\infty$ with $V(T')$ in that ball.
By \Cref{prop:approxHJBflow} the learned flow stays in the sample span, so it satisfies $\dot{\EV}=\mathcal T(\EV)-(\mathcal T-\widehat{\mathcal T})(\EV)$. Subtracting this equation from $\dot V=\mathcal T(V)$ leaves the exact linear part $\AbFPK-\rho\mathcal I$ acting on $e\defeq V-\EV$, and variation of constants gives, for $V_0=\EV_0=0$,
\begin{align}
e(t)
&=\int_0^t \mathrm e^{(t-\zeta)(\AbFPK-\rho\mathcal I)}\bigl[\mathcal D_\bu(\BbFPK V(\zeta),\cdot)
\notag\\
&\quad-\mathcal D_\bu(\BbFPK\EV(\zeta),\cdot)+(\mathcal T-\widehat{\mathcal T})(\EV(\zeta))\bigr]\,\mathrm d\zeta,
\label{eq:app:mild}
\end{align}
where the state reward cancels. Since $\mathcal D_\bu(\cdot,\bx)$ is globally Lipschitz, \eqref{eq:app:operator_difference} and \eqref{eq:app:hjb_operator_perturbation} bound the last term of \eqref{eq:app:mild} by $c\,\delta_{\mathcal T}$ on $[0,T']$. Since $\rho\mathcal I-\AbFPK$ is sectorial, \cite[Theorem~1.4.3]{henry1981geometric} gives $\tinynorm{\mathrm e^{t(\AbFPK-\rho\mathcal I)}}_{\LIN\to H_\mu^1}\le C_{\AbFPK}\,t^{-1/2}$, so
\begin{align}
\|e(t)\|_{H_\mu^1}
\le
\int_0^t\frac{C_{\AbFPK}\bigl(L_{\mathcal D}\norm{\BbFPK}\|e(\zeta)\|_{H_\mu^1}+c\,\delta_{\mathcal T}\bigr)}{\sqrt{t-\zeta}}\,\mathrm d\zeta.
\label{eq:app:gronwall}
\end{align}
The singular Gronwall inequality~\cite[Lemma~7.1.1]{henry1981geometric} applied to \eqref{eq:app:gronwall} then gives $\|V(T')-\EV(T')\|_{\LIN}=\|e(T')\|_{\LIN}\le\|e(T')\|_{H_\mu^1}=\mathcal O(\delta_{\mathcal T})$, so $\EV(T')$ lies in the same ball for $\delta_{\mathcal T}$ small. Since $\Vstar$ attracts the exact flow, $T'$ may be enlarged so that $V$ varies arbitrarily little on $[T',T'+\tau]$. The same $\mathcal O(\delta_{\mathcal T})$ bound at $T'+\tau$ then makes $\tinynorm{\EV(T'+\tau)-\EV(T')}_{\LIN}$, and hence the radius below, smaller than the radius of that ball, which places $\EVstar$ near $\Vstar$.

By \eqref{eq:app:c1_hjb_perturbation}, for $c\,\delta_{\mathcal T}$ below the exact stability margin, the learned time-$\tau$ map is a contraction with factor $\kappa<1$ there. It maps into itself the closed ball in the range of $\widehat{\mathsf E}_{\mathsf s}^*$ centered at $\EV(T')$ of radius $(1-\kappa)^{-1}\tinynorm{\EV(T'+\tau)-\EV(T')}_{\LIN}$, so Banach's fixed-point theorem gives a unique fixed point $\EVstar=\widehat{\mathsf E}_{\mathsf s}^*\bv^\star$ there. It is a zero of $\widehat{\mathcal T}$, since the learned flow is time-invariant and the fixed point is unique. Near $\EVstar$ the flow decays at the rate set by $\widehat D_F(\EVstar)$,
\[
\widehat\lambda_{\mathrm{gap}}+\rho
=-\max\{\operatorname{Re}z:z\in\sigma(\widehat D_F(\EVstar))\}>0,
\]
the spectrum being that of the Jacobian of $\bm T$ at $\bv^\star$. Absorbing the transient up to $T'$ into $\widehat C$ gives, for $\EV_0=\EV(0)$,
\begin{equation}
\|\EVstar-\EV_T\|_{\LIN}
\le
\widehat C e^{-(\widehat{\lambda}_{\mathrm{gap}}+\rho)T}
\|\EVstar-\EV_0\|_{\LIN},
\label{eq:app:convergence}
\end{equation}
for $\widehat C<\infty$. This decay rate controls every mode of the finite-dimensional flow and therefore also its policy-relevant quotient by constants. For $\EV_0=0$ and $T=k\Delta t$, \eqref{eq:app:convergence} gives the convergence term in \Cref{thm:end-to-end-error}.

\smallskip
\noindent\textbf{Infinite-horizon learning error.}
Let $\EVstar$ be the stationary point constructed above. Since $\widehat{\mathcal T}(\EVstar)=0$, \eqref{eq:app:operator_difference} and \eqref{eq:app:hjb_operator_perturbation} give
\begin{equation}
\|\mathcal T(\EVstar)\|_{\LIN}
=\|(\mathcal T-\widehat{\mathcal T})(\EVstar)\|_{\LIN}
\le c\,\delta_{\mathcal T},
\label{eq:app:residual_bound}
\end{equation}
where $c=\max\{1,\tinynorm{\EVstar}_{\spIN}\}$ is finite uniformly in $N$ by the hypothesis of \Cref{thm:end-to-end-error}.
Recall that $D_F^\star(\Vstar)=\frac{\partial}{\partial V}\mathcal T(\Vstar)$. Since $\mathcal T(\Vstar)=0$ and $\nabla_{\bl}\mathcal D_{\bu}$ is locally Lipschitz on the bounded costate set, a first-order Taylor expansion around $\Vstar$ gives, with the product estimate above,
\[
\mathcal T(\EVstar)=D_F^\star(\Vstar)(\EVstar-\Vstar)+\mathsf R,
\]
with $\|\mathsf R\|_{\LIN}\le\frac{L_{\nabla\mathcal D}}{2}\|\EVstar-\Vstar\|_{H_\mu^m}^2$, where $L_{\nabla\mathcal D}<\infty$ is that Lipschitz constant, taken together with $C_{\BbFPK}$, and depends locally on $\Vstar$.
It remains to invert $D_F^\star(\Vstar)$. \Cref{lem:main:expConvHJB} and the resolvent representation~\citep[Theorem~1.10]{engel2000one} give

\begin{align*}
[D_F^\star(\Vstar)]^{-1}&=-\int_0^\infty e^{tD_F^\star(\Vstar)}\,\mathrm dt,
\\
 \|[D_F^\star(\Vstar)]^{-1}\|_{\LIN\to\LIN}& \le\frac{\bar C}{\lambda_{\mathrm{gap}}+\rho}.
\end{align*}
where the rate is modulo constants, since $D_F^\star(\Vstar)\mathbf 1=-\rho\mathbf 1$ makes constants decay only at rate $\rho$. This bound is in $\LIN$, whereas the Taylor remainder is quadratic in $H_\mu^m$. The two are joined by the graph estimate $\tinynorm{h}_{H_\mu^m}\le C_{\mathrm{reg}}(\tinynorm{D_F^\star(\Vstar)h}_{\LIN}+\tinynorm{h}_{\LIN})$, which is part of the same regularity. Applying the resolvent to the Taylor identity, using \eqref{eq:app:residual_bound}, and absorbing $C_{\mathrm{reg}}$ into $\bar C$ gives
\begin{equation}
\|\EVstar-\Vstar\|_{H_\mu^m}
\le
\frac{\bar C}{\lambda_{\mathrm{gap}}+\rho}
\bigl(c\,\delta_{\mathcal T}
+\tfrac{1}{2}L_{\nabla\mathcal D}\|\EVstar-\Vstar\|_{H_\mu^m}^2\bigr).
\label{eq:app:local_learning_error}
\end{equation}
Choose the fixed neighborhood small enough that the quadratic term in \eqref{eq:app:local_learning_error} is at most one half of its left-hand side. Direct absorption and the continuous embedding $H_\mu^m\hookrightarrow\LIN$ then give
\begin{equation}
\|\EVstar-\Vstar\|_{\LIN}
\le
\frac{C'\delta_{\mathcal T}}{\lambda_{\mathrm{gap}}+\rho}
\le
\frac{C'\!\mathcal O(\delta)}{\lambda_{\mathrm{gap}}+\rho},
\label{eq:app:asymptGuaranteesThm}
\end{equation}
where $C'\in(0,\infty)$ absorbs the fixed local and regularity constants, and the second inequality uses $\delta_{\mathcal T}=\mathcal O(\delta)$ from \eqref{eq:app:hjb_operator_perturbation}.

\smallskip
\noindent\textbf{Discretization error.}
The IMEX step is a one-step method whose increment function is Lipschitz in $\bv$ and consistent of first order, so on the fixed interval $[0,T']$ the global error bound for one-step methods~\cite[Theorem~12.2]{suli2003introduction} gives $\mathcal O(\Delta t)$. The IMEX map has the same stationary point as the learned flow and is locally stable there for sufficiently small $\Delta t$. Hence errors after $T'$ are damped rather than accumulated, giving
\begin{equation}
\sup_{T=k\Delta t}\norm{\EV_T-\widehat V_k}_{\LIN}
=\mathcal O(\Delta t).
\label{eq:app:discretError}
\end{equation}
Finally, summing \eqref{eq:app:asymptGuaranteesThm}, \eqref{eq:app:convergence}, and \eqref{eq:app:discretError} yields \eqref{eq:end2enderror_rate}, after absorbing the finite learning and convergence prefactors into the single theorem constant $C$ and the discretization prefactor into the $\mathcal O(\Delta t)$ term.
\end{IEEEproof}

\begin{IEEEproof}[Proof of \Cref{cor:RateEst}]
For KRR, the operator-norm error decomposes into regularization bias and estimation error:
\begin{equation}
\begin{aligned}
\error(\widehat{\mathsf{G}}_\reg)
&\defeq
\norm{\target-\injectOUT\widehat{\mathsf{G}}_\reg}_{\spIN\to\LOUT} \\
&\le
\underbrace{\norm{\target-\injectOUT {\mathsf{G}}_\reg}_{\spIN\to\LOUT}}_{\text{regularization bias}}
+\underbrace{\norm{\injectOUT({\mathsf{G}}_\reg-\widehat{\mathsf{G}}_\reg)}_{\spIN\to\LOUT}}_{\text{estimation error}} .
\end{aligned}
\label{eq:app:error_decomp}
\end{equation}
Here ${\mathsf{G}}_\reg$ is the population KRR solution in \eqref{eq:main:population_krr}. We recall the three assumptions from~\cite{Kostic2023KoopmanEigenvalues} in the present notation.
\begin{enumerate}[label={\rm \textbf{(RC)}},leftmargin=7ex,nolistsep]
    \item\label{cond:app:RC}\emph{Regularity/source condition.} For some $\rpar\in(1,2]$ there exists $\rcon>0$ such that
    $\mathsf{C}_{\mathsf{z}\mathrm d}\mathsf{C}_{\mathsf{z}\mathrm d}^* \preceq \rcon^2 \mathsf{C}_{\mathsf{z}\mathsf{z}}^{1+\rpar}$.
\end{enumerate}
\begin{enumerate}[label={\rm \textbf{(BK)}},leftmargin=7ex,nolistsep]
    \item\label{cond:app:BK}\emph{Bounded kernels.} There exist $\bcon,\ccon>0$ such that
    $\operatorname*{ess\,sup}_{\bx\sim\mu}\|\featx(\bx)\|_{\spIN}^2\le \bcon$ and
    $\operatorname*{ess\,sup}_{\bm z\sim\eta}\|\featz(\bm z)\|_{\spOUT}^2\le \ccon$.
\end{enumerate}
\begin{enumerate}[label={\rm \textbf{(SD)}},leftmargin=7ex,nolistsep]
    \item\label{cond:app:SD}\emph{Spectral decay.} There exist $\spar\in(0,1]$ and $\scon>0$ such that
    $\lambda_j(\mathsf{C}_{\mathsf{z}\mathsf{z}})\le \scon j^{-1/\spar}$ for all $j\in\{1,2,\ldots\}$.
\end{enumerate}
Under these assumptions, the KRR rate of~\cite[Theorem~5.4]{kostic2024consistent} and \cite[Theorem~2]{Kostic2023KoopmanEigenvalues} gives the high-probability estimate
\[
\error(\widehat{\mathsf{G}}_\reg)
\le C_{\mathrm{KRR}} N^{-\frac{\rpar}{2(\rpar+\spar)}}\ln(\xi^{-1})
\]
when $\reg$ is chosen of order $N^{-\frac{1}{\rpar+\spar}}$.
Combining this estimate with \eqref{eq:main:block_to_joint_error} yields $\delta\le c_\eta C_{\mathrm{KRR}} N^{-\frac{\rpar}{2(\rpar+\spar)}}\ln(\xi^{-1})$. Substituting this estimate for $\delta$ in the learning term of \Cref{thm:end-to-end-error} yields
\[
\|\Vstar-\EV^\star\|_{\LIN}
\le C\,(\lambda_{\mathrm{gap}}+\rho)^{-1}
N^{-\frac{\rpar}{2(\rpar+\spar)}}\ln(\xi^{-1}),
\]
after absorbing constants, including $C_{\mathrm{KRR}}$, $c_\eta$, and the uniform Lipschitz constant of $\mathcal D_\bu(\cdot,\bx)$, into $C$.
\end{IEEEproof}

\begin{IEEEproof}[Proof of \Cref{cor:finite_difference_rate} and \Cref{rem:finite_difference_stepsize}]

Let $\EEstim_{\reg,\Delta\tau}$ and $\EEstim_{\reg}$ denote
the estimators obtained from finite-difference and exact
infinitesimal targets, respectively, at the same samples.
Each sample $i$ corresponds to an independent one-step
transition. Let $\widehat{\mathsf{C}}_{\mathsf{z}\mathrm d,\Delta\tau}$
and $\widehat{\mathsf{C}}_{\mathsf{z}\mathrm d}$ denote the empirical
analogues of the population cross-covariance
$\mathsf{C}_{\mathsf{z}\mathrm d}$, obtained by replacing the expectation
under $\eta$ with the sample average and using the finite-difference
and exact targets, respectively. Their difference is
\[
\widehat{\mathsf{C}}_{\mathsf{z}\mathrm d,\Delta\tau}
-\widehat{\mathsf{C}}_{\mathsf{z}\mathrm d}
=
\tfrac{1}{N} \textstyle \sum_{i=1}^{N}
\mathsf L_i
\bigl(
\bm d_{\Delta\tau}^{(i)}-\bm d^{(i)}
\bigr),
\]
where, for $\bm w\in\R^{n_s}$, $
\mathsf L_i\bm w
\defeq
\featz(\bz^{(i)})\otimes
\sum_{k=1}^{n_s}
w_k\partial_{s_k}\featx(\bx^{(i)}).
$
By \textbf{(BK)}, compactness of $\supp\eta$, and
$k\in C^{2,2}$, the state--action and derivative-feature
factors are uniformly bounded. Hence
$\tinynorm{\mathsf L_i}_{\R^{n_s}\to\HS{\spIN,\spOUT}}$ is uniformly bounded. For simplicity of presentation, we temporarily omit the
sample index $(i)$. Define the fixed-action drift
$\bm d^{\bu}(\bx)\defeq\bm f(\bx)+\bm G(\bx)\bu$, and let
$\bX_t$ be the exact solution initialized at $\bx$, with
$\bu$ held fixed. The integral form of the SDE gives
\[
\bm d_{\Delta\tau}-\bm d
=
\underbrace{
\frac{1}{\Delta\tau}
\int_0^{\Delta\tau}
\left[
\bm d^{\bu}(\bX_t)-\bm d^{\bu}(\bx)
\right]\mathrm dt
}_{\bm{\mathfrak r}}
+
\sqrt{\frac{2\epsilon}{\Delta\tau}}\,\bm g,
\]
where
$\bm g\defeq W_{\Delta\tau}/\sqrt{\Delta\tau}
\sim\mathcal N(0,I)$.

\textbf{Bounding the discretization bias.}
Writing $\ctrlbFPK^{\bu}$ for the generator with $\bu$ fixed,
Dynkin's formula
\cite[Theorem~7.4.1]{oksendal2003stochastic}, applied
componentwise, gives
\[
\expect[\mathfrak r_k\mid\bz]
=
\frac{1}{\Delta\tau}
\int_0^{\Delta\tau}\int_0^t
\expect\!\left[
\bigl(
\ctrlbFPK^{\bu}d_k^{\bu}
\bigr)(\bX_\zeta)
\middle|
\bz
\right]
\mathrm d\zeta\,\mathrm dt.
\]
 The additional assumption that the conditional first moment of $\ctrlbFPK^{\bu}d_k^{\bu}$ is uniformly bounded over the considered one-step intervals for every drift component
$k\in[n_s]$ therefore gives
$\tinynorm{\expect[\bm{\mathfrak r}\mid\bz]}
\le c\Delta\tau$.
Combined with the bounded $\tinynorm{\mathsf L_i}$, this shows
that the conditional-bias contribution to
$\widehat{\mathsf{C}}_{\mathsf{z}\mathrm d,\Delta\tau}
-\widehat{\mathsf{C}}_{\mathsf{z}\mathrm d}$
is $O(\Delta\tau)$.

\textbf{Bounding the centered stochastic contribution.}
Global Lipschitzness of $\bm f$ and $\bm G$, together with
bounded sampled actions and initial states, gives a uniform
finite-horizon second-moment bound for the fixed-action SDE.
Using its integral form, Cauchy--Schwarz, and
$\expect\tinynorm{W_t}^{2}\le ct$, we obtain
\[
\begin{aligned}
\expect\!\left[
\norm{\bX_t-\bx}^{2}
\middle|
\bz
\right]
&\le
2t\int_0^t
\expect\!\left[
\norm{\bm d^{\bu}(\bX_\zeta)}^{2}
\middle|
\bz
\right]\mathrm d\zeta
+
4\epsilon\expect\norm{W_t}^{2}
\\
&\le c(t^2+\epsilon t).
\end{aligned}
\]
Jensen's inequality and Lipschitz continuity of
$\bm d^{\bu}$ consequently yield
$
\expect\!\left[
\tinynorm{\bm{\mathfrak r}}^{2}
\middle|
\bz
\right]
\le
c(\Delta\tau^2+\epsilon\Delta\tau).
$
Restoring the sample index, each sample $i$ has its own
one-step solution $\bX^{(i)}_t$ started at $\bx^{(i)}$ and
driven by an independent standard Wiener process $W^{(i)}$,
with normalized increment
$\bm g_i\defeq W^{(i)}_{\Delta\tau}/\sqrt{\Delta\tau}
\sim\mathcal N(0,I)$. Let
$\widetilde{\bm{\mathfrak r}}_i
\defeq
\bm{\mathfrak r}_i
-\expect[\bm{\mathfrak r}_i\mid\bz^{(i)}]$
and
$\bm Z\defeq(\bz^{(1)},\ldots,\bz^{(N)})$.
Conditional on $\bm Z$, the summands below are independent
and centered. Since conditional variance is bounded by the
conditional second moment and
$\expect\tinynorm{\bm g_i}^{2}\le c$, using $\tinynorm{a+b}^{2}\le2\tinynorm a^{2}+2\tinynorm b^{2}$, we have
\[
\begin{aligned}
&\expect\!\left[
\norm{
{\textstyle\frac{1}{N}}
\mathop{\textstyle\sum}\limits_{i=1}^{N}
\mathsf L_i
\left(
\widetilde{\bm{\mathfrak r}}_i
+
\sqrt{{\textstyle\frac{2\epsilon}{\Delta\tau}}}\,\bm g_i
\right)
}_{\mathrm{HS}}^{2}
\middle|
\bm Z
\right]
\\
&\qquad\le
{\textstyle\frac{c}{N}}
\left(
\Delta\tau^2
+
\epsilon\Delta\tau
+
{\textstyle\frac{\epsilon}{\Delta\tau}}
\right)
\le
c\left(
\Delta\tau^2
+
{\textstyle\frac{\epsilon}{N\Delta\tau}}
\right).
\end{aligned}
\]
where the last inequality uses $0<\Delta\tau\le1$ and
$N\ge1$. The conditional Markov's inequality, together with the
$O(\Delta\tau)$ bias bound, therefore gives
$
\tinynorm{\widehat{\mathsf{C}}_{\mathsf{z}\mathrm d,\Delta\tau}
-\widehat{\mathsf{C}}_{\mathsf{z}\mathrm d}}_{\HS{\spIN,\spOUT}}
=
O_{\mathbb P}\!\left(
\Delta\tau+
\sqrt{\frac{\epsilon}{N\Delta\tau}}
\right).
$
Finally, the empirical state-action covariance is unchanged,
so $\EEstim_{\reg,\Delta\tau}-\EEstim_{\reg}
= \ECreg^{-1}
(\widehat{\mathsf{C}}_{\mathsf{z}\mathrm d,\Delta\tau}
-\widehat{\mathsf{C}}_{\mathsf{z}\mathrm d})$. After applying the Tikhonov resolvent bound~\cite[Section~5.1]{engl1996regularization} we obtain
\[
\norm{
\EEstim_{\reg,\Delta\tau}-\EEstim_{\reg}
}
=
O_{\mathbb P}\!\left(
\reg^{-1}
\left[
\Delta\tau+
\sqrt{\frac{\epsilon}{N\Delta\tau}}
\right]
\right).
\]
Combining this perturbation with the exact-target rate from
\Cref{cor:RateEst}, the block-identifiability relation, and
the learning term in \Cref{thm:end-to-end-error}, and substituting
$\reg$ of order $N^{-1/(\rpar+\spar)}$, proves
\eqref{eq:main:finite_difference_rate}.

\end{IEEEproof}
\section*{ACKNOWLEDGMENT}
To improve clarity and coherence, parts of this manuscript have been grammatically revised with the assistance of Grammarly and ChatGPT. After using these tools, the authors reviewed and edited the content as needed and take full responsibility for the publication's content.

	\bibliographystyle{IEEEtran}
	\bibliography{references_clean}

\end{document}